\documentclass{article}

\usepackage{microtype}
\usepackage{graphicx}
\usepackage{subfigure}
\usepackage{booktabs} 

\usepackage{hyperref}

\usepackage[accepted]{icml2023}

\usepackage{amsfonts}
\usepackage{amsmath}
\usepackage{amssymb}
\usepackage{amsthm}
\usepackage{graphicx}
\usepackage{float}
\usepackage{xcolor}
\usepackage{bm}
\usepackage{quiver}
\usepackage{soul}

\usepackage[capitalize,noabbrev]{cleveref}

\theoremstyle{plain}
\newtheorem{theorem}{Theorem}[section]

\theoremstyle{definition}

\theoremstyle{remark}

\newcommand\restr[2]{{
  \left.\kern-\nulldelimiterspace 
  #1 
  \littletaller 
  \right|_{#2} 
  }}
\newcommand{\littletaller}{\mathchoice{\vphantom{\big|}}{}{}{}}

\usepackage[textsize=tiny]{todonotes}

\icmltitlerunning{Rigid Body Flows for Sampling Molecular Crystal Structures}

\newcommand{\change}[1]{#1}

\begin{document}

\twocolumn[
\icmltitle{Rigid Body Flows for Sampling Molecular Crystal Structures}



\icmlsetsymbol{equal}{*}

\begin{icmlauthorlist}
\icmlauthor{Jonas Köhler}{equal,msr,fucs}
\icmlauthor{Michele Invernizzi}{equal,fucs}
\icmlauthor{Pim de Haan}{qc,uva}
\icmlauthor{Frank Noé}{msr,fucs,fuph,ruch}
\end{icmlauthorlist}

\icmlaffiliation{msr}{Microsoft Research AI4Science}
\icmlaffiliation{fucs}{Freie Universität Berlin, Department of Mathematics and Computer Science}
\icmlaffiliation{fuph}{Freie Universität Berlin, Department of Physics}
\icmlaffiliation{ruch}{Rice University, Department of Chemistry}
\icmlaffiliation{qc}{Qualcomm AI Research, an initiative from Qualcomm Technologies, Inc.}
\icmlaffiliation{uva}{University of Amsterdam}

\icmlcorrespondingauthor{Frank, Noé}{franknoe@microsoft.com}

\icmlkeywords{normalizing flows, Boltzmann generators, molecular crystals, molecular dynamics, free energy, generative models}

\vskip 0.3in
]



\printAffiliationsAndNotice{\icmlEqualContribution} 

\begin{abstract}
Normalizing flows (NF) are a class of powerful generative models that have gained popularity in recent years due to their ability to model complex distributions with high flexibility and expressiveness.
In this work, we introduce a new type of normalizing flow that is tailored for modeling positions and orientations of multiple objects in three-dimensional space, such as molecules in a crystal. 
Our approach is based on two key ideas: first, we define smooth and expressive flows on the group of unit quaternions, which allows us to capture the continuous rotational motion of rigid bodies; second, we use the double cover property of unit quaternions to define a proper density on the rotation group.
This ensures that our model can be trained using standard likelihood-based methods or variational inference with respect to a thermodynamic target density.
We evaluate the method by training Boltzmann generators for two molecular examples, namely the multi-modal density of a tetrahedral system in an external field and the ice XI phase in the TIP4P water model. Our flows can be combined with flows operating on the internal degrees of freedom of molecules, and constitute an important step towards the modeling of distributions of many interacting molecules.
\end{abstract}

\section{Introduction}

\begin{figure}[t]
    \centering
    \includegraphics[width=\linewidth]{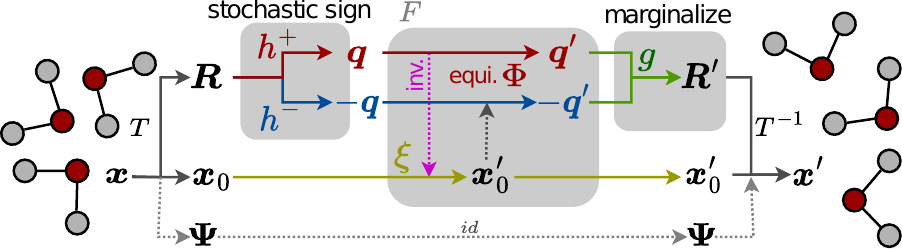}
    \caption{Flow on SO(3) via the $S^3$ double cover: we first transform a system of Cartesian coordinates $\bm x$ via $T$ into the triple $(\bm x_{0}, \bm R, \bm \Psi)$, containing the global translation $\bm x_{0}$, the global rotation $\bm R$ and its fixed inner degrees of freedom $\bm \Psi$ (e.g., bond lengths and inner angle of water molecules). While keeping $\bm \Psi$ fixed we transform the poses as follows: first map $\bm R$ onto one of the two representing quaternions $\bm q$ or $-\bm q$ stochastically (here we picked $\bm q$ as indicated in red). Then transform $(\bm x_{0}, \bm q)$ into $(\bm x'_{0}, \bm q')$ using a sign-flip equivariant coupling flow $F$ made of position updates $\xi$ and rotation updates $\Phi$. As the double cover projection $g$ maps both $\bm q'$ and $-\bm q'$ to the same rotation element $\bm R'$, transforming back corresponds to a simple marginalization over both stochastic paths. We finally obtain $\bm x'$ using the inverse rigid body transform $T^{-1}$.}
    \label{fig:double-cover-method}
\end{figure}

Normalizing flows (NF) \cite{tabak2010density, rezende2015variational, papamakarios2019normalizing} are popular deep learning generative models, that have been applied to the physical sciences in a variety of ways, such as for sampling lattice models \cite{nicoli2020asymptotically, nicoli2021estimation, li2018neural, boyda2021sampling, albergo2019flow}, approximating the equilibrium density of molecular systems \cite{noe2019boltzmann, kohler2020equivariant, wu2020snf, xu2021learning}, and estimating free energy differences \cite{wirnsberger2020targeted, ding2021deepbar}.
\change{
This success is motivated by the fact that, contrary to other generative models, NF are built to provide an efficient reweighting scheme that allows for exact sampling from a given energy-based probability distribution.
In this work we present a novel NF architecture that is particularly suited to sample molecular crystals, i.e. periodic systems where multiple copies of the same molecule are arranged in a lattice structure.
Molecular crystals are of key importance for several applications, ranging from the pharmaceutical industry to solar energy production \cite{bernstein2020polymorphism}.
For this reason, there is great interest in developing efficient and effective methods to predict the physical properties of molecular crystals through computer simulations rather than expensive laboratory experiments.
}

\paragraph{\change{Normalizing flows for equilibrium sampling and free-energy difference estimation}}
In this work, we focus on methods that have a primary application in the field of Boltzmann Generators (BG) \cite{noe2019boltzmann}. BGs are generative models that are trained to sample conformations of molecules in equilibrium. These follow a Boltzmann-type distribution, $\mu(\bm x) \propto \exp(-u(\bm x))$. Here $u$ is the dimensionless potential defined by the molecular system and the thermodynamic state in which it is simulated, such as the canonical ensemble at a certain temperature. 
BGs \change{are primarily implemented using NFs and as such} can be trained using a combination of maximum likelihood estimation on potentially biased data obtained from molecular dynamics (MD) simulations, and energy-based training via the reverse Kullback-Leibler (KL) divergence.
Once trained, BGs can be used for importance sampling \cite{noe2019boltzmann, Muller2019}, as efficient proposals in Markov chain Monte Carlo (MCMC) applications \cite{sbailo2021neural, Gabrie2022}, or as teacher models when learning coarse-grained MD force-fields \cite{koehler2023}.
\change{
An important advantage of BGs over traditional methods such as MD and MCMC, is that the latter struggle to efficiently sample systems characterized by long-lived metastable states separated by very low-probability transition regions.
Such multistable systems are ubiquitous, e.g. chemical reactions, conformational rearrangements in biomolecules, phase transitions in materials.
One of the most important property for these systems, is the relative stability of their metastable states, i.e. their free energy difference.
Boltzmann generators and NF can be used to directly connect such metastable states and compute the free energy difference \cite{wirnsberger2020targeted,Rizzi2021,invernizzi2022}.
}

\paragraph{Building flows on natural molecular representations}

Despite their potential, previous flow architectures for molecular systems have severe limitations. Most importantly, many high-fidelity models rely on representations that are either non-scalable (e.g., global internal coordinates \cite{kohler2021smooth, koehler2023, invernizzi2022}), non-transferable (e.g., principal components \cite{noe2019boltzmann}), or unnatural (e.g., splits between Cartesian axes \cite{wirnsberger2020targeted, wirnsberger2022}).
Equivariant all-atom representations, while more principled, require computationally intensive and approximate methods such as solving neural ODEs \cite{kohler2020equivariant,satorras2021n} and can be challenging to scale and integrate with energy-based training.

A more natural and scalable representation would place atoms relative to the orientation and position of a chemical entity such as the molecule or residue, thus separating inter-molecular degrees of freedom and intra-molecular placement of atoms relative to the pose.
This becomes especially important in solvated systems and molecular crystals, where the most interesting emergent properties result from inter-molecular interactions, whereas intra-molecular degrees of freedom vary often predictably. 

Additionally, when simulating such systems, it is common practice to fix the stiffest internal degrees of freedom, typically inter-atomic bonds and angles, as it allows for larger integration time steps and thus faster mixing without giving up much accuracy.
Yet, if rigid residues are present in a simulation, the molecular density $\mu$ becomes non-singular only for a sub-manifold of the full space. Thus, any NF approach modeling all degrees of freedom can neither be reweighted against such a singular target density nor can it be trained by minimizing the reverse KL divergence.

Transforming the molecular pose independently from inner degrees of freedom requires a physically meaningful normalizing flow architecture on the pose manifold that scales to systems composed of many interacting poses.


\paragraph{Contributions}
Here, we present such an approach to designing normalizing flows for sampling the joint distribution over positions and orientations of systems composed of many molecules. As handling the internal degrees of freedom separately was extensively studied in prior work, such as \cite{kohler2021smooth}, we focus on the important limit case of purely rigid bodies in this work.
Our flow architecture is ideal for molecular simulation applications due to its unique traits, including:
\begin{itemize}
    \item They prescribe fully smooth densities on the rigid body sub-manifold. The smoothness of the flow density has shown to be critical for faithful modeling of physical force fields \cite{kohler2021smooth, koehler2023}.
    \item They are compatible with permutation equivariant architectures, such as used in \citet{wirnsberger2022}. This feature has been shown to be critical when scaling flows to larger bulk systems, e.g., extended crystals or water boxes.
    \item External pose and inner degrees of freedom are treated independently. As such they are fully compatible with prior work focusing on modeling the internal degrees of freedom separately.
\end{itemize}

As part of the method we further contribute two new smooth flow architectures for the rotation manifold $SO(3)$, namely symmetrized Moebius transformations and projective convex gradient maps. We finally demonstrate the efficacy and efficiency of the method by sampling the multi-modal density of a tetrahedral body in an external field, as well as sampling ice XI crystals following the TIP4P water model at different sizes and temperatures with high accuracy.

\section{Related Work}

\paragraph{Normalizing flows on manifolds}
Normalizing flows on manifolds have been extensively studied for Riemannian geometry, e.g., in the form of convex potential flows \cite{cohen2021, rezende2021implicit} or neural ODEs \cite{chen2018neural} on manifolds
\cite{lou2020neural, Katsmann2021, mathieu2020riemannian, falorsi2021continuous, Ben-Hamu2022}.
Approaches to smooth coupling flows on non-trivial manifolds, like tori and spheres, were discussed in \citet{rezende2020normalizing, kohler2021smooth}.
Beyond that there exist approaches to non-smooth normalizing flows on Riemannian submanifolds via charts \cite{gemici2016normalizing,kalatzis2021density}. Using the double cover with normalizing flows to estimate densities of single poses in the context of computer vision was concurrently discussed in the work of \citet{liu2023delving}. We discuss the relation of this concurrent work to the present work in Sec. \ref{sec:flip-equi-diffeos}.

\paragraph{Density estimation on SO(3)}
Beyond flows other methods for neural density estimation on $SO(3)$ have been proposed for domains outside of molecular physics: \citet{falorsi2019reparameterizing} discussed using the exponential map to push-forward densities on the Lie-algebra $\mathfrak{so}(3)$ to $SO(3)$. Furthermore, \citet{murphy2021} proposed single pose estimation using the double cover via an implicit neural representation on $S^3$. 

\paragraph{Sampling equilibrium structures with normalizing flows}

Normalizing flows for sampling molecular systems were studied in the context of importance sampling \cite{wu2020snf, kohler2020equivariant, Dibak2021, kohler2021smooth, Midgley2022} and estimating free energy differences \cite{noe2019boltzmann, wirnsberger2020targeted, ding2021deepbar, Rizzi2021, wirnsberger2022,invernizzi2022, Ahmad2022, Coretti2022}. Furthermore, \citet{satorras2021n} used NF for generating conformations across molecular space, however only focusing on density estimation without explicit treatment of the thermodynamics.

\paragraph{Sampling molecular crystals without machine learning}
Established methods to sample molecular crystals typically rely on molecular dynamics or Markov chain Monte Carlo simulations \cite{FrenkelBook}.
Several protocols have been proposed to compute free energy difference and phase diagrams for molecular crystals, one of the most popular being thermodynamic integration \cite{Frenkel1984,Vega2008}.
Relevant to our purposes is the targeted free energy perturbation method \cite{Jarzynski2002} that has been combined with the multistate Bennett acceptance ratio (MBAR) \cite{Shirts2008mbar} to compute crystal free energies \cite{Schieber2018, Schieber2019}.

\section{Theory \& Method}

\change{
In molecular equilibrium sampling we are provided with a (dimensionless) potential function
\begin{align}
    u(\bm x) \colon \mathbb{R}^{N \times 3} \rightarrow \mathbb{R}
\end{align}
describing the probability density of samples $\mu(\bm x)$ in thermodynamic equilibrium, via the relation
\begin{align}
    \mu(\bm x) = \frac{\exp(-u(\bm x))}{Z},
\end{align}
where $Z$ is the unknown normalizing constant (partition function) of the system.

In the more general case, we further assume a parametric ensemble of potentials $u_{\alpha}$ with corresponding densities $\mu_{\alpha}$, where each $\alpha$ corresponds to a different thermodynamic state, e.g., pressure or temperature.

Primary goals are now
\begin{enumerate}
    \item \label{bullet:goal-one} drawing asymptotically unbiased i.i.d. samples from $\mu$ from which we can estimate expectation values of downstream observables, as well as,
    \item estimating the log-ratio 
    \begin{align}
    \Delta F = -\log\left(Z_{\alpha_1} / Z_{\alpha_{0}} \right)
    \end{align}
    between the partition functions $Z_{\alpha_{0}}$ and $Z_{\alpha_{1}}$ of two thermodynamic states $\alpha_{0}$ and $\alpha_{1}$. This quantity is also called \emph{free energy difference} between $u_{\alpha_{0}}$ and $u_{\alpha_{1}}$ and is an important measure telling which state is more stable among different thermodynamic conditions.
\end{enumerate}

Previous work \cite{noe2019boltzmann, wirnsberger2020targeted} has shown that NFs are a natural choice for such tasks, as they allow to formulate the sampling problem relative to a tractable base density, while providing asymptotic guarantees on unbiasedness.

NF approximate a target density $\mu$ by a parametric diffeomorphic map, $\Phi(\cdot; \bm \theta)$, that transforms samples from a base density, $\bm z \sim p_{0}(\bm z)$, into samples that follow the push-forward density, 

\begin{align}
\Phi_{*}(p_{0})(\bm x; \bm \theta) := p_{0}\left(\Phi^{-1}(\bm x; \bm \theta)\right) \left|\bm J^{-1}_{\Phi}(\bm x; \bm \theta) \right|
\end{align}

In standard applications, like density estimation on images, the parameters $\bm \theta$ are learned by minimizing the negative log-likelihood on data 
\begin{align}
    \label{eq:maximum-likelihood}
    \bm \theta_{ML} = \arg\min_{\theta} \mathbb{E}_{\bm x \sim \mu(\bm x)} \left[- \log \Phi_{*}(p_{0})(\bm x; \bm \theta)\right].
\end{align}

In the molecular sampling setup, however, samples from $\mu(\bm x)$ are usually sparse and biased (e.g. when obtained from non-converged MD simulations). As such, training can combine biased likelihood training with minimizing the reverse Kullback-Leibler divergence

\begin{align}
    \label{eq:reverse-kl-divergence}
    \bm \theta_{KL} = \arg\min_{\theta} D_{KL}\left[ \Phi_{*}(p_{0})(\cdot; \bm \theta) \| \mu(\cdot) \right].
\end{align}

If the base density $p_{0}$ is given by a simple density with closed-form sampling algorithm, e.g., an isotropic Gaussian, these model can be turned into asymptotically unbiased independence samplers using reweighing techniques, such as self-normalized importance sampling \cite{noe2019boltzmann}. These models where coined \emph{Boltzmann generators} in prior work and can be used to tackle goal \ref{bullet:goal-one}.

Alternatively, we could consider $p_{0} = \mu_{\alpha_{0}}$, for some reference potential $u_{\alpha_{0}}$ and learn the mapping to any other potential $u_{\alpha_{1}}$ via reverse KL minimization. This leads to the method of \emph{learned free energy perturbation} (LFEP) \cite{wirnsberger2020targeted} and allows to directly estimate \begin{align}
    \label{eq:free-energy-upper-bound}
    \Delta F \leq D_{KL}\left[ \Phi_{*}(\mu_{\alpha_0})(\cdot; \bm \theta) \| \mu_{\alpha_{1}}(\cdot) \right]
\end{align}
from above. In practice one can, e.g., run MD on $\mu_{\alpha_{0}}$ to obtain samples, train the flow using loss \ref{eq:reverse-kl-divergence}, and then estimate the bound \ref{eq:free-energy-upper-bound}.
}

\paragraph{Normalizing flows for rigid bodies}
\change{We now explain how this framework can be used when studying systems composed of rigid bodies}

Let us assume we have a system $\bm X \in \mathbb{R}^{N \times K \times 3}$ consisting of $N$ bodies, each consisting of $K$ beads in $\mathbb{R}^{3}$. In the present model, we assume that these bodies are rigid, i.e., we only sample the joint translation or rotation of the $K$ beads. However, combining this rigid body flow with a flow operating on internal degrees of freedom is conceptually straightforward.

Each rigid body $\bm x = (\bm x_{0}, \ldots, \bm x_{K-1})$ can equivalently be described as a triple $(\bm x_{0}, \bm R, \bm \Psi)$ of the position $\bm x_{0} \in \mathbb{R}^{3}$ of the first bead, a rotation matrix $\bm R \in SO(3) \subset \mathbb{R}^{3 \times 3}$, and $3K-6$ inner degrees of freedom $\bm \Psi$. 
An intuitive approach to modeling a diffeomorphism in this representation is keeping $\bm \Psi$ fixed and only describing transformations of $(\bm x_{0}, \bm R)$.

While modeling flows on $\mathbb{R}^{3}$ is clear, modeling smooth normalizing flows directly on $SO(3)$ is challenging and has pros and cons depending on the representation. 
Rotation matrices are the natural way to handle rotational degrees of freedom, but designing expressive normalizing flows can be difficult due to the orthonormality and sign constraint.
Euler angles on the other hand show the gimbal lock phenomenon leading to a non-smooth representation. They furthermore act nonlinearly and induce a volume change. 
Another drawback of working on $SO(3)$ directly is that equivariance under rotations is very difficult to incorporate into the map. As shown, e.g., in \citet{kohler2020equivariant, satorras2021n} this can become critical when scaling flows to larger physical systems. 
While there exist intrinsic manifold approaches, such as \citet{falorsi2021continuous, mathieu2020riemannian,lou2020neural}, those require expensive and possibly inexact numerical integration methods and as such are hard to scale. Furthermore, computing their exact density is usually avoided via stochastic approximation of the divergence term which in practice is not sufficient for accurate reweighting of molecular systems \cite{kohler2020equivariant}.

Here we give a formulation of normalizing flows for rigid bodies that are smooth, fast to compute and invert, compatible with equivariance constraints, and provide a tractable exact density.

\subsection{Flows on $SO(3)$ via the $S^3 \rightarrow SO(3)$ double cover.}

Instead of working on $SO(3)$ directly, we can also model rotations via the group of \textit{unit quaternions}, also known as SU(2), i.e., the set of unit vectors $\bm q \in S^3 \subset \mathbb{R}^4$ equipped with the quaternion product. We denote the quaternion product between two quaternions $\bm q_{1}$ and $\bm q_{2}$ as $\bm q_{1} \odot \bm q_{2}$,  and the conjugation of $\bm q$ as $\bm q^*$.
Let $\iota \colon \mathbb{R}^3 \hookrightarrow \mathbb{R}^{4}$ be the canonical embedding where we embed a point $\bm x \in \mathbb{R}^{3}$ as a purely imaginary quaternion $(x_{0}, x_{1}, x_{2}, 0)$, and $\pi \colon \mathbb{R}^{4} \twoheadrightarrow \mathbb{R}^3$ be the corresponding projection, such that $\pi \circ \iota = \mathrm{id}$. For any $\bm q \in S^3$ the map $\bm R_{\bm q} \colon \mathbb{R}^{3} \rightarrow \mathbb{R}^{3}, \bm x \mapsto \bm q . \bm x := \pi\left( \bm q \odot \iota(\bm x) \odot \bm q^{*} \right) $ is a rotation of $\bm x$ around the origin. This defines a smooth map $g \colon S^{3} \rightarrow SO(3), \bm q \mapsto \bm R_{\bm q}$. Furthermore, $g$ is surjective and as such, each rotation $\bm R$ can be represented by some $\bm q \in g^{-1}(\bm R)$. However, $g$ is not injective as we have $\bm q.\bm x = (- \bm q). \bm x$. In fact, $g$ forms a covering map, i.e., for each $\bm R \in SO(3)$ there is an open neighborhood $U_{\bm R}$, such that $g^{-1}(U_{\bm R}) \approx U_{\bm R} \times \mathbb{Z}_{2}$.
Furthermore for each $\bm R$ there are locally defined smooth functions $h_{\bm R}^{+}, h_{\bm R}^{-} \colon U_{\bm R} \rightarrow S^3$, such that, $g \circ h_{\bm R}^{+} = g \circ h_{\bm R}^{-} = \mathrm{id}_{U_{\bm R}}$ and $h_{\bm R}^{+} = - h_{\bm R}^{-}$. Abusing notation we will abbreviate $\bm q(\bm R) := h_{\bm R}^{+}(\bm R)$ and $-\bm q(\bm R) := h_{\bm R}^{-}(\bm R)$ in the following discussion.

\paragraph{Stochastic paths over the double cover}
\label{sec:smooth-rigid-body-flows}

We can leverage this smooth double cover to design smooth flows for rigid bodies in the following way (see Fig. \ref{fig:double-cover-method}):
\begin{itemize}
    \item We first transform $\bm x$ into its rigid representation $(\bm x_0, \bm R, \bm \Psi)$. This can be done, e.g., by computing the pose $(\bm x_0, \bm R)$, removing it from $\bm x$, and then computing $\bm \Psi$ in this standard frame. We keep $\bm \Psi$ fixed and write the transformation relative to $\bm \Psi$ as $T_{\bm \Psi} \colon \bm x \mapsto (\bm x_0, \bm R)$.
    \item We then stochastically embed $\bm R$ as either $\bm q(\bm R)$ or $-\bm q(\bm R)$ with equal probability. This results in two possible paths through the transformation.
    \item Given a diffeomorphism $F \colon \mathbb{R}^{3} \times S^{3} \rightarrow \mathbb{R}^{3} \times S^{3}$ we transform $F(\bm x_{0}, \bm q) = (\bm x'_{0}, \bm q')$.
    \item Now by using the double cover $g$, we obtain $\bm R' = g(\bm q')$. It is important to note, that due to the double cover property, we would also obtain $\bm R' = g(-\bm q')$.
    \item We can then invert the rigid body transform to get $(\bm x'_{0}, \bm R', \bm \Psi) \mapsto \bm x'$ by using $T^{-1}_{\bm \Psi}$
\end{itemize}
As we explain in the following paragraph, explicit sampling the sign of $\bm q(\bm R)$ is not even necessary and you can choose either sign arbitrarily. 
\change{For completeness, we note that the lift from SO(3) to $S^3$ could be implemented by any stochastic function, as long as it remains the identity when composed with the projection map. This includes deterministic sign choices as the limiting case. We elaborate on that in appendix \ref{a:covering_flows}.}

\paragraph{Volume change induced by the transformation}
Now let us equip the inputs $\bm x$ with a base density $p_{0}$. Furthermore, consider the case where $F$ is invariant under sign flips of $\bm q$ in the first coordinate and equivariant in the second, i.e., we assume that for all $\bm x_{0} \in \mathbb{R}^{3}, \bm q \in S^{3}$:
\begin{align}
    F(\bm x_{0}, \bm q) = (\bm x_0', \bm q')
    \Rightarrow F(\bm x_{0}, -\bm q) = (\bm x_0', - \bm q').
\end{align}
In this situation, we can compute the total volume change induced by the transformation 
\begin{align}
    \bm x \rightarrow (\bm x_{0}, \pm \bm q(\bm R)) \rightarrow (\bm x'_{0}, \bm q') \rightarrow\bm x',
\end{align}
independent of the choice of path, as 
\begin{align}
    \label{eq:volume-change}
   \left|\bm J_{\bm x \rightarrow \bm x'}(\bm x) \right| &= \left|\bm J_{F}(\bm x_{0}, \bm q(\bm R)) \right| = \left|\bm J_{F}(\bm x_{0}, -\bm q(\bm R)) \right|.
\end{align}
This stems from the fact, that the volume contribution of each path is identical and that each path produces exactly the same rotation element at its end. Additionally, the volume change introduced by $T_{\bm \Psi}$ and $T^{-1}_{\bm \Psi}$ cancels out. We give a formal derivation of this transformation law in appendix \ref{a:volume-change-on-manifolds}.

\paragraph{Constructing a flip-symmetric map}

Given a class of flip-equivariant diffeomorphisms $\Phi(\cdot; \bm \theta) \colon S^{3} \rightarrow S^{3}$ and arbitrary diffeomorphisms $\xi(\cdot; \bm \theta) \colon \mathbb{R}^{3} \rightarrow \mathbb{R}^{3}$ we can construct $F$ via the coupling layers \cite{dinh2014nice, dinh2016rnvp}:
\begin{align}
    \bm x_{0}' = \xi(\bm x_{0}; \bm \theta_{\xi}(\bm q)) \qquad \bm q' = \Phi(\bm q; \bm \theta_{\Phi}(\bm x_{0}')).
    \label{eq:coupling-layers}
\end{align}
This map is flip-symmetric according to the previous paragraph as long as we assert that $\bm \theta_{\xi}$ is a flip-invariant conditioning map. We can efficiently invert $F$ and evaluate its volume change according to Eq. \eqref{eq:volume-change}, whenever $\xi$ and $\Phi$ are easy to invert and their change of volume can efficiently be computed. While there exist many candidates that could be chosen for $\xi$, it is less obvious how to design a suitable family $\Phi$.

\subsection{Flip-equivariant diffeomorphisms on $S^3$} \label{sec:flip-equi-diffeos}

As such, we introduce two classes of smooth and flip-equivariant diffeomorphisms on $S^d$ that can be used to realize $F$ in practice: symmetrized Moebius transforms and projective convex gradient maps.
While the first has analytic formulas to compute its volume change and inverse, it is less expressive if one aims to model very multi-modal target densities. As such we consider it as the flip-equivariant $S^3$ analog of the broadly used real-NVP \cite{dinh2016rnvp} layers.
The second requires more numerical effort to compute its inverse and volume change while being in principle arbitrarily expressive in modeling flip-symmetric multi-modal densities on $S^3$. We consider it as the flip-equivariant $S^3$ analog to recently introduced convex-potential flows \cite{huang2021convex}.

\paragraph{Symmetrized Moebius transforms}

A generalized Moebius transform on $S^d$ can be given by
\begin{align}
    \Phi_{M}(\bm p; \bm q)
    &= \bm p - 2 \mathrm{proj}_{\bm q - \bm p}(\bm p) \label{eq: moebius-projection}
\end{align}
with $\mathrm{proj}_{\bm u}(\bm v) = \frac{\bm v^T \bm u}{\|\bm u\|^2_2}\bm v$.
Whenever $\|\bm q\| < 1$ the map $\Phi_{M}(\cdot; \bm q)$ defines a diffeomorphism on $S^d$ \cite{rezende2020normalizing, Kato_2020}. We can see this using the following intuition: first, send a ray from $\bm p$ through $\bm q$ until it intersects $S^3$ again at some point $\bm p'$. Then mirror $\bm p'$ onto $-\bm p'$ to get the final result in \eqref{eq: moebius-projection}.
Each such $\Phi_{M}(\cdot; \bm q)$ defines an involution $\Phi_{M}(\cdot; \bm q) \circ \Phi_{M}(\cdot; \bm q) = \textrm{id}_{S^d}$.
Unfortunately, only $\Phi_{M}(\cdot; \bm 0) = \textrm{id}_{S^d}$ is (trivially) flip-equivariant. However we can use $\Phi_{M}$ to construct the following family of flip-equivariant maps \footnote{See \url{https://www.geogebra.org/m/j7gpwcnf} for an interactive animation.}:
\begin{align}
    \Phi_{SM}(\bm p; \bm q) = \frac{\Phi_{M}(\bm p; \bm q) + \Phi_{M}(\bm p; -\bm q)}{\| \Phi_{M}(\bm p; \bm q) + \Phi_{M}(\bm p; -\bm q) \|}.
\end{align}
If $\|\bm q\| < 1$ each $\Phi_{SM}(\cdot; \bm q)$ defines a diffeomorphism on $S^d$ with an analytic inverse. Furthermore, there exists an analytic formula to compute its induced change of volume (see appendix \ref{a:volume_moebius}).

\paragraph{Projective convex gradient maps}
Another construction can be given as follows. Let $\phi \colon \mathbb{R}^{d+1} \rightarrow \mathbb{R}$ be a strictly convex and smooth map, with minimizer $\bm 0$. Then the normalized gradient map $\Phi \colon S^{d} \rightarrow S^{d}$, given by
\begin{align}
    \Phi_{CG}(\bm x) = \frac{\nabla_{\bm x} \phi(\bm x)}{\|\nabla_{\bm x} \phi(\bm x)\|^2},
\end{align} defines a diffeomorphism on $S^d$ (see proof in appendix \ref{a:diff_convex_map}). Furthermore, if $\phi$ is flip-invariant, the resulting map $\Phi_{CG}$ will be flip-equivariant. While we could model $\phi$ with arbitrarily complex convex functions, e.g., deep input-convex neural networks \cite{amos2017}, this requires us to compute the inverse and induced volume change using numeric methods. For $S^{3}$ however, we can compute the volume change in closed form efficiently (see appendix \ref{a:volume_convex_map}).

\paragraph{Relation to \citet{liu2023delving}}

A concurrent approach to modeling flows on $SO(3)$ via the double cover was pursued in  \citet{liu2023delving}. Our work differs significantly in the following aspects:
\begin{itemize}
    \item While we study physical systems composed of multiple rigid bodies following a joint pose distribution, this prior work studies estimating a single pose in the context of computer vision tasks. 
    
    \item In order to describe flexible distributions on $SO(3)$ they introduce two flows. The first requires a fiber bundle construction within a specified frame to apply non-smooth spline flows. This construction is fundamentally incompatible with rotational equivariance when modeling multiple poses jointly. The choice of a frame together with the non-smoothness of the induced density makes the approach unsuitable for physical systems as studied in this work.
    \item The second flow introduced in \citet{liu2023delving} is an affine $S^3 \rightarrow S^3$ flow, and, as we show in appendix \ref{a:liu_convex_map}, is a special case of our projective convex gradient maps.
    \item Interleaving these two flows requires many changes of coordinates between the $SO(3)$ matrices and the $S^3$ double cover. On the other hand, our Moebius and the projective convex gradient map can model smooth complex multi-modal densities without ever leaving the double-cover construction.
\end{itemize}

\section{Experiments}
\begin{figure}[t]
    \centering
    \includegraphics[width=1.0\linewidth]{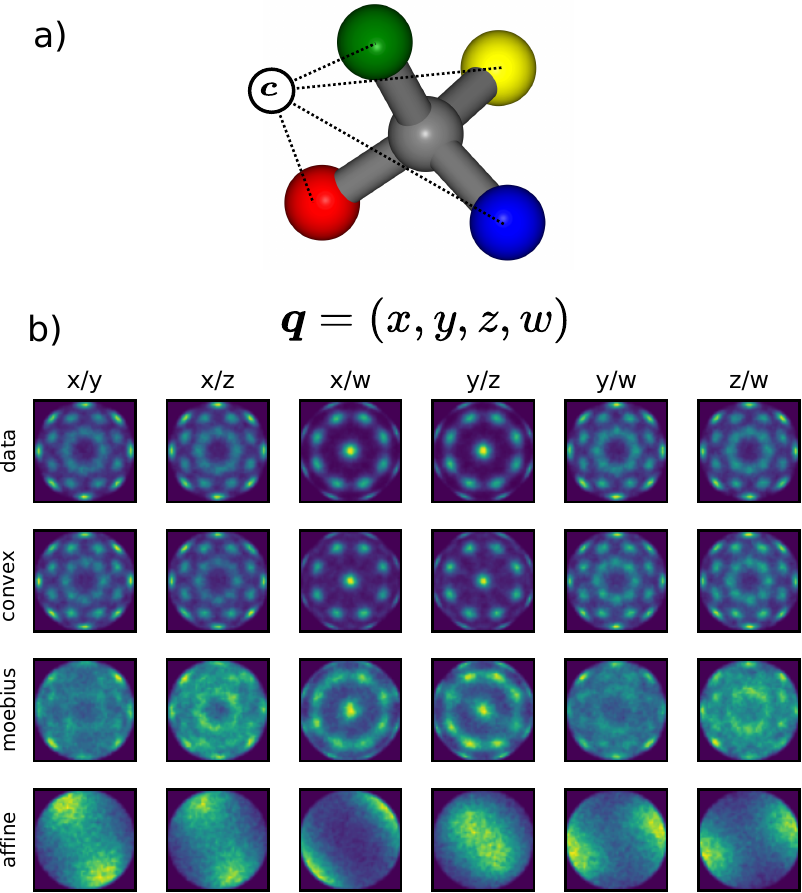}
    \caption{
    Tetrahedron in an external field: a) each colored bead is attracted with the same force towards the external point $\bm c$ according to the potential defined in Eq. \eqref{eq:toy-potential}. b) Density of rotational degree of freedom. Rotations are represented as unit-quaternions $\bm q = (x, y, z, w)$. First row: density of MD trajectory. Rows 2 to 4: densities of flows using projective convex gradient maps, symmetrized Moebius projections, and affine transformations, respectively.
    }
    \label{fig:tetrahedron-result}
\end{figure}
\begin{figure*}[t]
    \centering
    \includegraphics[width=\textwidth]{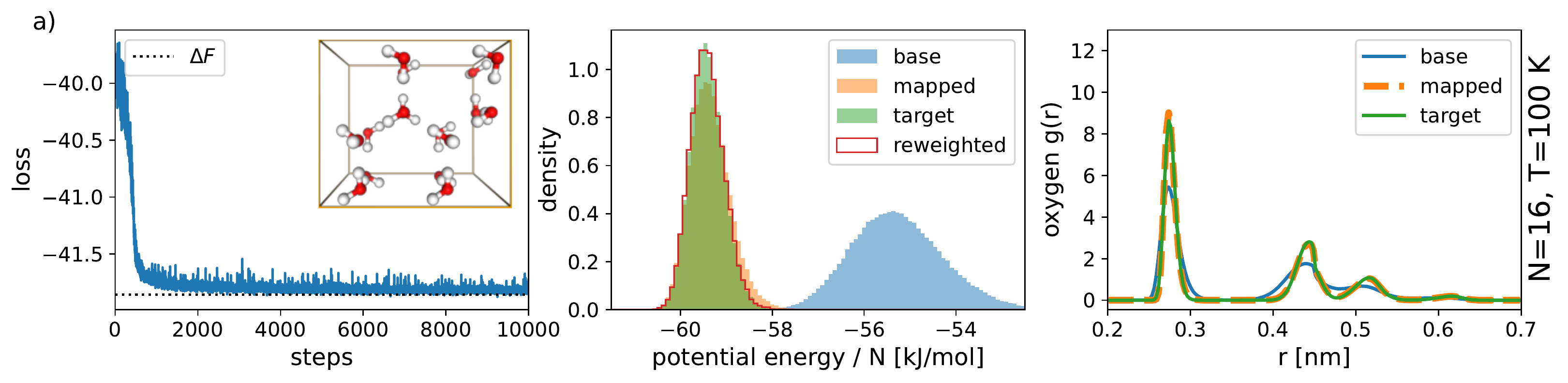}
    \includegraphics[width=\textwidth]{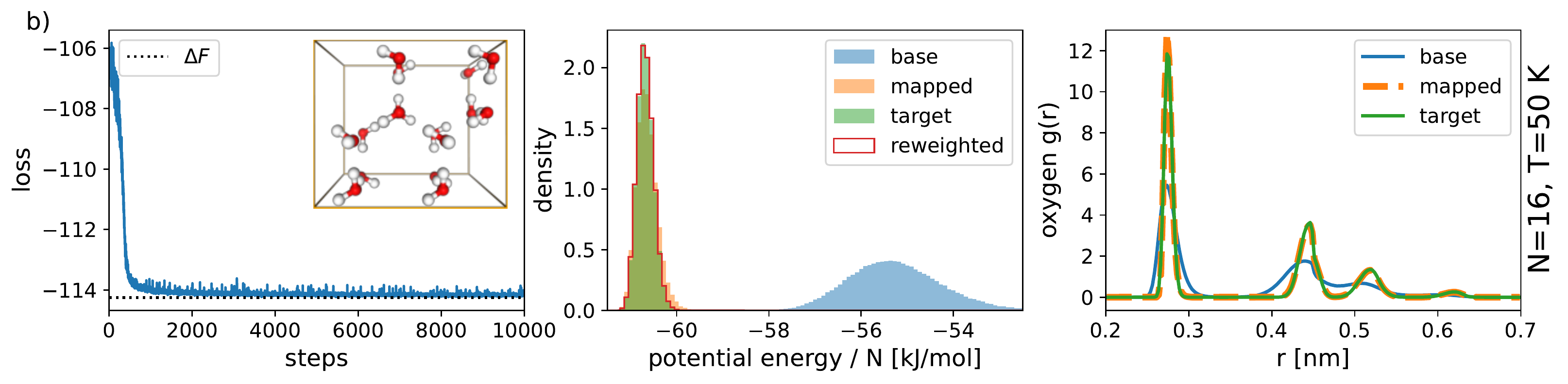}
    \includegraphics[width=\textwidth]{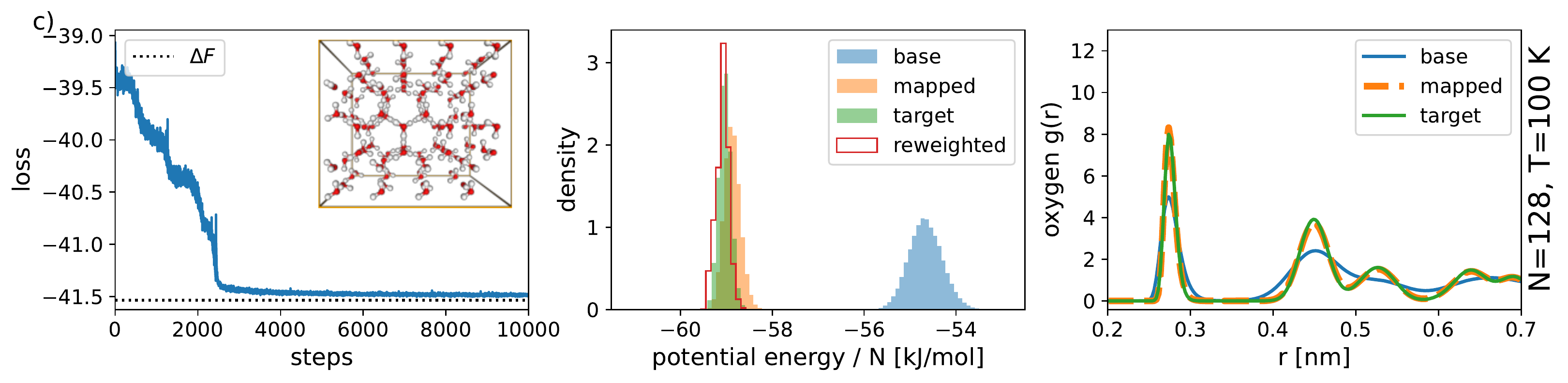}
    \caption{Results for the ice model for different target densities. 
    The temperature of the base density is always T$_0$ = 250~K, while the temperature T of the target is reported in the figure, as is the number of water molecules N.
    As expected, the loss approaches from above the reference per-molecule free energy difference $\Delta F$.
    The estimates reported in Table~\ref{table:df-results} are instead obtained by applying the LFEP estimator to the evaluation dataset.
    In all three cases, the map learned by the NF can transform the base density into the target one, as can be seen from the energy distribution and the oxygen-oxygen radial distribution function $g(r)$.
    The target distribution is not used for training, it is reported merely as a reference.
    }
    \label{fig:ice-results}
\end{figure*}

We show the efficacy of our method for rigid bodies in molecular physics by applying it to sampling two benchmark systems. The first system consists of a single rigid body following a very multi-modal density over the rotations. It serves as a benchmark to see how well the different flow architectures can express multi-modality.
The second system consists of an actual molecular crystal in different thermodynamic states. 

\subsection{Tetrahedron in external field}

Our first test system is given by a CH$_4$ (methane) molecule consisting of 5 atoms (see Fig. \ref{fig:toy-model-flow} a)). We keep internal degrees of freedom constant and fix the carbon to the origin so that only rotational motion is possible. This system interacts with an external field of the form 
\begin{align}
    \label{eq:toy-potential}
    u(\bm x) = C \cdot \sum_{k=1}^{5} \sum_{d=1}^{3}(x_{kd} - c_{d})^{4},
\end{align}
where $\bm c \in \mathbb{R}^{3}$ and $C > 0$ are control parameters.
If $\bm c \neq \bm 0$ there are multiple local minima. When sampled in equilibrium, e.g., using an MD simulation, this gives rise to a smooth and multi-modal density on the rotation manifold (see Fig. \ref{fig:tetrahedron-result} b) - top row).

As our goal is to model smooth densities on $SO(3)$ we compare the following three models on this system: the affine quaternion flow from \cite{liu2023delving} and the two transforms introduced in Sec. \ref{sec:flip-equi-diffeos}. 
Other flow models that could be considered are either not smooth, are not defined on the sub-manifold of interest, or do not possess a scalable way to compute exact densities and as such are not suitable for modeling molecular densities. Thus, we do not consider them in this comparison.


We generate an equilibrium dataset of 50,000 samples by sampling from $\mu(\bm x) \propto \exp(-u(\bm x))$ using OpenMM \cite{eastman2017openmm} and evaluate the different flow layers by their capability to match the multi-modal structure of its quaternion density when being trained by maximum-likelihood training.

Prior work \cite{huang2020augmented, koehler2023, chen2020vflow, wu2020snf, nielsen2020survae} showed that the expressivity of flows trained on multi-modal datasets can be increased by adding auxiliary Gaussian noise dimensions. We follow this approach and train each of the three tested flow methods on this data set accordingly by minimizing the variational bound to the negative log-likelihood until convergence (see appendix \ref{a:methane} for details on architecture and training).

Our results, as depicted in Fig. \ref{fig:tetrahedron-result}, show that projective convex gradient maps can faithfully reconstruct the multi-modal density of the rotations. While the symmetrized Moebius transforms are able to visibly resolve multiple modes, they clearly struggle with the strong multi-modality of the data. The affine quaternion layers can only represent a single mode and as such fail to represent the distribution faithfully. The parameterization of the projective convex potential can be critical for the expressivity of the flow. For brevity, we elaborate on this aspect in appendix \ref{a:methane_param}.

\subsection{Ice XI in the TIP4P water model}
As an example of a molecular crystal, we use water which, while being a simple molecule, is both of primary interest for MD simulations and exhibits highly nontrivial phase behavior \cite{Bore2022, Kapil2022}.
\change{ In this experiment, we aim to estimate the free energy difference $\Delta F$ between two different thermodynamic conditions, namely a reference temperature $T_{0}$ and a lower temperature $T$.}
For our simulations we investigate the hydrogen-ordered crystal phase of water, ice XI \cite{Matsumoto2021}, with the TIP4P-Ew rigid water model \cite{Horn2004}. We sample the canonical ensemble, thus fixed number of particles, volume, and temperature. This simple model system does not include quantum mechanical effects and cannot be expected to reproduce experimental measurements \cite{Abascal2005}. However, it is a useful setup to study the Boltzmann distributions of molecular crystals. 

The base density $\mu_{0}$ is given at temperature T$_{0}$ = $250\,$K, thus $u_{0}(\bm x) = (k_B T_0)^{-1}U(\bm x)$, where $k_B$ is the Boltzmann constant and $U(\bm x)$ is the TIP4P-Ew force-field energy. We then try to match the density of the same system at a different target temperature T. 
The quantity $\Delta F$ grows when the temperature gap increases, or when increasing the number of particles at a fixed temperature difference.
As such, we test our method for the following target potentials:
\begin{itemize}
    \item For a system composed of $N = 16$ water molecules we estimate $\Delta F$ for target temperatures T = $100\,$K and T = $50\,$K.
    \item For a system composed of $N = 128$ water molecules we estimate $\Delta F$ for a target temperature T = $100\,$K.
\end{itemize}

As reference, we compute the estimate of $\Delta F$ from MD simulations using the multistate Bennett acceptance ratio (MBAR) \cite{Shirts2008mbar} \change{(see appendix \ref{a:intro_mbar})}.
This method requires an overlap in phase space between the distributions that we want to calculate the free energy difference of. 
Thus we need to run multiple additional MD simulations at a ladder of intermediate temperatures which can quickly become expensive when $\Delta F$ is large \cite{invernizzi2022}.

By training a NF, we can instead use a single MD run to sample the base, and then use the LFEP estimator to compute $\Delta F$.
This can result in a considerable reduction of the computational cost, see appendix \ref{a:ice_cost}.
As proposed by \citet{Rizzi2021} we split the base MD run into two parts, one for training and the other one for the LFEP evaluation, to avoid systematic errors.
The used flow consists of coupling layers between positions and rotations according to Fig.~\ref{fig:double-cover-method}. We present it in detail in appendix \ref{a:ice_flow}.

\paragraph{Results}
The results are shown in Fig.~\ref{fig:ice-results}.
We show that we can achieve a close overlap of the energy distributions between the mapped density from our flow and the target density as obtained by reference MD simulations (second column). We can furthermore reweight the energies and the oxygen-oxygen radial distribution function to achieve nearly perfect overlap (third column).

The $\Delta F$ per molecule (thus divided by N) is reported in Table \ref{table:df-results}. We estimated the error via bootstrapping and report it given as two standard deviations. 
Other approaches like LBAR \cite{wirnsberger2020targeted, wirnsberger2022} could provide a more accurate estimate but would require samples from the target distribution which we do not assume to be available in our experiments.

\begin{table}[t]
\caption{Estimates of the free energy difference $\Delta F$ per molecule obtained with molecular dynamics (MBAR) and with our normalizing flow (LFEP).}
\label{table:df-results}
\vskip 0.15in
\begin{center}
\begin{small}
\begin{sc}
\begin{tabular}{lccc}
\toprule
Target &  MBAR & LFEP \\
\midrule
N=16, T=100~K    & -41.857 $\pm$ 0.007& -41.859 $\pm$ 0.002 \\
N=16, T=50~K & -114.251 $\pm$ 0.007& -114.252 $\pm$ 0.005 \\
N=128, T=100~K    & -41.535 $\pm$ 0.002& -41.534 $\pm$ 0.003 \\
\bottomrule
\end{tabular}
\end{sc}
\end{small}
\end{center}
\vskip -0.1in
\end{table}

\section{Discussion}

In this work, we presented a new approach to approximate the densities of multiple interacting molecules by modeling their positions and orientations using normalizing flows. A key element of this was a derivation of a smooth flow structure using the quaternion double cover and providing an efficient implementation via two categories of flip-equivariant flows on $S^3$. We furthermore demonstrated the effectiveness of this approach for modeling densities of molecular crystals by evaluating it on a multi-modal benchmark system and a range of ice systems.

\change{
We note that beyond the very important application to molecular crystals, rigid body flows could also become relevant in other domains, such as robotics, evidenced by related work like \citet{brehmer2023edgi}.
}

\paragraph{Limitations and possible extensions}

While the result for ice XI is promising and paves the way for many interesting applications of normalizing flows in the field of molecular crystals, a major challenge ahead is dealing with phase transitions or even going beyond the crystal phase to liquid and gas. However, these are still open problems even in the case of non-molecular systems, e.g., when monatomic crystals are modeled \citet{wirnsberger2022}. An interesting but nontrivial next step would be extending the present architecture with a flow model for the positions that can handle fluids and phase transitions.

A second aspect that we did not explore further in this work is exploiting the $SE(3)$ symmetry of jointly moving all rigid bodies.
\change{
Both introduced flow layers can easily be extended to fully rotation equivariant architectures by making the learnable functions $\xi, \theta_\xi, \theta_\Phi$ in Eq. \eqref{eq:coupling-layers} equivariant, with architectures such as EGNN \cite{satorras2021n}, NequIP \cite{batzner20223} or MACE \cite{batatia2022mace}. Such architectures can also compute pairwise interactions equivariant to jointly moving pairs of rigid bodies.}
Furthermore, many rigid bodies have internal symmetries, such as the mirror symmetry of the water molecule. For $N$ water molecules, this gives a symmetry group of order $2^N$. To scale to larger systems, built-in equivariance to this group may be necessary.

It is important to note that while here we only consider rigid-body molecules, the proposed flow architecture can be straightforwardly extended to incorporate the internal degrees of freedom of the molecules. This is an important aspect, as it is essential to handle larger molecules or more accurate force fields.

Finally, recent work of \citet{Abbott2022} raised questions about the scaling limits of normalizing flows when sampling physical potentials in lattice physics. 
Although such a study has not yet been carried out for molecular systems, it will be important to understand how this result relates to the sampling of molecular crystals and whether flow-based approaches can be reliably and efficiently scaled to much larger systems.

\section*{Software and Data}
All the code used to obtain the results is available at \url{https://github.com/noegroup/rigid-flows}.


\section*{Acknowledgements}
We thank Andreas Krämer for his invaluable editorial support in preparing this version of the manuscript and for his insightful advice. We furthermore thank Maaike Galama for helpful discussions about MBAR.

J.K and F.N. acknowledge funding by DFG CRC1114 Project B08, DFG RTG DAEDALUS, ERC consolidator grant 772230. M.I. acknowledges support from the Humboldt Foundation for a Postdoctoral Research Fellowship.




\bibliography{example_paper}
\bibliographystyle{icml2023}

\newpage
\appendix
\onecolumn

\change{
\section{The sampling problem in molecular crystals} \label{a:intro_mbar}

Molecular crystals are of great interest for several important applications, but there are many open problems when it comes to efficiently characterize their properties via computer simulations.
One of the reasons is that there are an exponentially high number of energetically stable polymorphs, but at any given thermodynamic condition only few are stable enough to be observed experimentally, and only one is the most stable one.
Even for a simple molecule like water, more than 20 crystal polymorphs have been observed 
\footnote{\url{https://en.wikipedia.org/wiki/Ice}}.
This poses several challenges, that are typically tackled with different methodologies.
Here we do not focus on how to find all the energetically stable polymorphs, but rather on the stability at non-zero temperature, where entropic effects are important and thus free energy differences must be estimated, instead of just energy differences.

One of the most popular ways of computing free energy differences for atomic and molecular crystals, is thermodynamic integration (TI) \cite{FrenkelBook}.
As an example, let us consider two states, A and B, with energies $u_A(\mathbf{x})$ and $u_B(\mathbf{x})$. 
To estimate the free energy difference $\Delta F_{AB}$ with TI, one has to run multiple MD or MCMC simulations of the system along an interpolation between $u_A$ and $u_B$, such that each simulation samples a region of the phase space $\mathbf{x}$ that has some overlap with the closest ones.
In the example considered in our paper, A and B are simply two different temperatures and the interpolation is done by slowly changing the temperature, but one can also perform TI between a physical state and a reference ideal normal distribution, usually referred as Einstein crystal in the literature \cite{FrenkelBook}.
Once these simulations have been performed, the actual estimate of $\Delta F_{AB}$ can be obtained with various postprocessing methods, but a typical choice is the MBAR method, that has been proved to provide the lowest variance estimator \cite{Shirts2008mbar}.

Possibly, the most straightforward way of estimating $\Delta F_{AB}$ is to use the free energy perturbation formula \cite{Zwanzig1954}:
\begin{equation}
    \Delta F_{AB} = -\log \langle e^{u_A-u_B} \rangle_A = \log \langle e^{u_B-u_A} \rangle_B
\end{equation}
which requires samples either from A or B.
It is important to notice that a good overlap in configuration space is crucial, otherwise the variance of the ensemble average is orders of magnitude larger than $\Delta F_{AB}$.
Sampling a ladder of overlapping intermediate states allows one to use this formula to estimate $\Delta F_{AB}$ one step at the time.
The MBAR method is based on the same idea, but uses a self consistent procedure to combine all the samples and obtain an estimator which minimizes the variance \cite{Shirts2008mbar}.

The main drawback of TI is that it can be computationally extremely expensive, requiring sampling from several intermediate states that are of no direct interest.
This is especially exacerbated in the case of molecular crystals, where the integration path can be highly nontrivial and where to obtain accurate potential energies one must perform expensive quantum mechanical calculations.
To avoid this expensive calculation, \citet{Jarzynski2002} proposed to use an explicit invertible map to bridge A and B, the so called targeted free energy perturbation method.
Defining such maps is far from trivial, even for the simplest systems, that is why the method has rarely been used.
However, recently \citet{wirnsberger2020targeted} proposed to use normalizing flows to learn such maps, which gave rise to the LFEP method that is used also in this work.
}

\section{Proofs and derivations}

\subsection{Volume change on manifolds}
\label{a:volume-change-on-manifolds}
We follow the notation of \citet{rezende2020normalizing} Appendix A.
Let $\mathcal{M}$ and $\mathcal{N}$ be $m$ and $n$ dimensional submanifolds of $\mathbb{R}^{d}$ respectively and $F\colon \mathcal{M} \rightarrow \mathcal{N}$ a smooth injective map that can be extended to open neighborhoods of $\mathcal{M}$ and $\mathcal{N}$. Then we can compute its induced change of volume as follows: let $T_{\bm x}\mathcal{M}$ and $T_{\bm F(\bm x)}\mathcal{N}$ be the tangent spaces at a point $\bm x \in \mathcal{M}$ and its image $ F(\bm x) \in \mathcal{N}$. 
Furthermore, let $\bm E_{\bm x} \in \mathbb{R}^{d\times n}$ and $\bm E_{F(\bm x)} \in \mathbb{R}^{d\times m}$ be bases of $T_{\bm x} \mathcal{M}$ and $T_{\bm F(\bm x)} \mathcal{N}$ respectively. Then
\begin{align}
    |\bm J_{F}(\bm x)| = \sqrt{\det \bm E_{\bm x}^T \bm J_{F}^T(\bm x) \bm J_{F}(\bm x) \bm E_{\bm x}}.
\end{align}
If $m=n$ we also have
\begin{align}
    |\bm J_{F}(\bm x)| = \det \bm E_{F(\bm x)}^T \bm J_{F}(\bm x) \bm E_{\bm x}.
\end{align}

\subsection{Density of the mixture}
Let $p_{0}$ be a density over the inputs $\bm x$.
Furthermore, define
\begin{align}
    A_{+}(\bm x_{0}, \bm R) = (\bm x_{0}, \bm q(\bm R)), ~ 
    A_{-}(\bm x_{0}, \bm R) = (\bm x_{0}, \bm -q(\bm R)), 
\end{align}
Summing up the probabilities of each path and accounting for the induced volume change of each transformation, we obtain the mixture density
\begin{align}
    \label{eq:full-mixture-density}
    p(\bm x') =& 
    \tfrac{1}{2} \cdot | \bm J_{T_{\bm \Psi}}(\bm x')| \cdot |\bm J_{F^{-1}}(A_{+}(T_{\bm \Psi}(\bm x')))|  \cdot | \bm J_{T^{-1}_{\bm \Psi}}(F^{-1}(A_{+}(T_{\bm \Psi}(\bm x')))| \cdot  p_{0}(T_{\bm \Psi}^{-1}(F^{-1}(A_{+}(T_{\bm \Psi}(\bm x'))) ))\nonumber \\ 
    &+ \tfrac{1}{2} \cdot | \bm J_{T_{\bm \Psi}}(\bm x')| \cdot |\bm J_{F^{-1}}(A_{-}(T_{\bm \Psi}(\bm x')))|  \cdot | \bm J^{-1}_{T_{\bm \Psi}}(F^{-1}(A_{-}(T_{\bm \Psi}(\bm x')))| \cdot  p_{0}(T_{\bm \Psi}^{-1}(F^{-1}(A_{-}(T_{\bm \Psi}(\bm x'))) )) 
\end{align}

First, we see the following: after $T_{\bm \Psi}$ maps $\bm \Psi$ into the standard frame any change in $\bm x_{0}$ or $\bm R$ is merely a SE(3) action and as such does not contribute to the volume. From that, we get that 
$| \bm J_{T_{\bm \Psi}}(\bm x')| \cdot | \bm J_{T^{-1}_{\bm \Psi}}(F^{-1}(A_{\pm}(T_{\bm \Psi}(\bm x')))| = 1.$

Now let $F$ be flip-symmetric according to the definition in Sec. \ref{sec:flip-equi-diffeos}. 

From the definition of $F$ and using the double cover we get $T_{\bm \Psi}^{-1}(F^{-1}(A_{-}(T_{\bm \Psi}(\bm x'))) ) = T_{\bm \Psi}^{-1}(F^{-1}(A_{+}(T_{\bm \Psi}(\bm x'))) )$.

We furthermore get
\begin{align}
    \bm J_{F}(\bm x, -\bm q) = \underbrace{\left[ \begin{matrix}
        \bm I_{3 \times 3} & \bm 0\\
        \bm 0 & - \bm I_{4 \times 4}
    \end{matrix} \right]}_{\bm B :=} \bm J_{F}(\bm x, \bm q).
\end{align}
Let $\bm E$ be a basis according to Sec. \ref{a:volume-change-on-manifolds}. Then we immediately see that 
\begin{align}
    |\bm J_{F}(\bm x, \bm q)| 
    &= \sqrt{\det \bm E^T \bm J_{F}(\bm x, \bm q)^T \bm J_{F}(\bm x, \bm q) \bm E} \\
    &= \sqrt{\det \bm E^T \bm J_{F}(\bm x, \bm q)^T \bm B^T \bm B \bm J_{F}(\bm x, \bm q) \bm E} \\
    &= \sqrt{\det \bm E^T \bm J_{F}(\bm x, -\bm q)^T\bm J_{F}(\bm x, -\bm q) \bm E} \\
    &= |\bm J_{F}(\bm x, -\bm q)| 
\end{align}

Combining this we can simplify Eq. \eqref{eq:full-mixture-density} into 

\begin{align}
    p(\bm x') =& 
    \tfrac{1}{2} \cdot |\bm J_{F^{-1}}(A_{+}(T_{\bm \Psi}(\bm x')))| \cdot  p_{0}(T_{\bm \Psi}^{-1}(F^{-1}(A_{+}(T_{\bm \Psi}(\bm x'))) ))\nonumber \\ 
    &+ \tfrac{1}{2} \cdot |\bm J_{F^{-1}}(A_{-}(T_{\bm \Psi}(\bm x')))|  \cdot  p_{0}(T_{\bm \Psi}^{-1}(F^{-1}(A_{-}(T_{\bm \Psi}(\bm x'))) )) \\ 
    =& 
    \tfrac{1}{2} \cdot |\bm J_{F^{-1}}(A_{+}(T_{\bm \Psi}(\bm x')))| \cdot  p_{0}(T_{\bm \Psi}^{-1}(F^{-1}(A_{+}(T_{\bm \Psi}(\bm x'))) ))\nonumber \\ 
    &+ \tfrac{1}{2} \cdot |\bm J_{F^{-1}}(A_{+}(T_{\bm \Psi}(\bm x')))|  \cdot  p_{0}(T_{\bm \Psi}^{-1}(F^{-1}(A_{+}(T_{\bm \Psi}(\bm x'))) ))\\
    =& |\bm J_{F^{-1}}(A_{+}(T_{\bm \Psi}(\bm x')))| \cdot  p_{0}(T_{\bm \Psi}^{-1}(F^{-1}(A_{+}(T_{\bm \Psi}(\bm x'))) )) 
\end{align}
by substituting $\bm x' = \bm x'(\bm x)$ and using the short-hand notation $
    \left|\bm J_{F}(\bm x_{0}, \bm q(\bm R)) \right| = \left|\bm J_{F}(A_{+}(T_{\bm\Psi}(\bm x')) \right|
$
we end up with the formula for the induced volume change:
\begin{align}
   p_{0}(\bm x) = p(\bm x'(\bm x)) \left|\bm J_{F}(\bm x_{0}, \bm q(\bm R)) \right|.
\end{align}



\subsection{Derivations for symmetrized Moebius transform}
A generalized Moebius transform $S^d \to S^d$, given a parameter $\bm q \in B^d = \{\bm q \in \mathbb R^{d+1}  \mid |\bm q| < 1\}$ can be given by
\begin{align}
    \Phi_{M}(\bm p; \bm q)
    &= \bm p - 2 \mathrm{proj}_{\bm q - \bm p}(\bm p) 
\end{align}
with $\mathrm{proj}_{\bm u}(\bm v) = \frac{\bm v^T \bm u}{\|\bm u\|^2_2}\bm v$.
This map is an involutive diffeomorphism with $\Phi_M(\Phi_M(\bm p; \bm q); \bm q)= \bm p$.
The sign-symmetrized variant is 
\begin{align}
    \Phi_{SM}(\bm p; \bm q) = \frac{\Phi_{M}(\bm p; \bm q) + \Phi_{M}(\bm p; -\bm q)}{\| \Phi_{M}(\bm p; \bm q) + \Phi_{M}(\bm p; -\bm q) \|}.
\end{align}
which satisfies $\Phi_{SM}(\bm p;\bm q)=\Phi_{SM}(\bm p; -\bm q)$.

Both maps are equivariant to the choice of orthonormal coordinates, as for any coordinate transformation $g \in O(d)$, $\Phi_M(g \bm p; g \bm q) = g \Phi_M(\bm p; \bm q)$ and then by linearity $\Phi_{SM}(g \bm p; g \bm q) = g \Phi_{SM}(\bm p; \bm q)$.
Combined with the $\bm q \mapsto -\bm q$ invariance of $\Phi_{SM}$, this implies to the desired sign-equivariance: $\Phi_{SM}(-\bm p;\bm q)=-\Phi_{SM}(\bm p;\bm q)$.

Due to the $O(d)$ equivariance, we're free to analyze the maps in a convenient coordinate system. Without loss of generality, we can place point $\bm p \in S^d$ at $(x, \sqrt{1-x^2}, 0, 0, ...)$ and $\bm q \in B^d$ at $(r, 0, 0, ...)$. It is easy to see that the image $\Psi_{SM}(\bm p; \bm q)=(x', y', 0, 0, ....)$ has zeros in all but the first two coordinates. Also, the coordinates are given by the planar $d=1$ version of the transformation: $(x', y')=\Phi_{SM}((x, \sqrt{1-x^2}), (r, 0))$.

Using a computer algebra system, we can simplify this expression to find:
\begin{align*}
x' &= \frac{x(r^2-1)}{\sqrt{1+r^4+r^2(2-4x)}} \\
y'
&=-\sqrt{1-{x'}^2}
\end{align*}

\paragraph{Invertibility}
Given $r$, the map $x \mapsto x'$ can be inverted via a computer algebra system to find:
\begin{align*}
x = \frac{-x'(r^2+1)}{\sqrt{1 + r^4 + r^2(4{x'}^2-2)}}
\end{align*}
By assumption, the $y$ coordinate was in the positive half-plane, so $y=\sqrt{1-x^2}$.

To compute the general inverse $\bm p = \Phi_{SM}^{-1}(\bm p'; \bm q)$, we compute a $g \in O(d)$ so that $g \bm q=(r, 0, ...)$ and $g \bm p' = (x', -\sqrt{1-{x'}^2}, ...)$. Then we use the above inverse to find $g \bm p = (x, \sqrt{1-x^2}, 0, ...)$ and find $\bm p = g^{-1}(x, \sqrt{1-x^2}, 0, ...)$.


\paragraph{Change of volume}\label{a:volume_moebius}
For the change of volume of the symmetric Moebius transformation, we focus on the case $d=3$ of the three-sphere. Now, a convenient parametrization (without loss of generality) is $\bm p = (1, 0, 0, 0)$ and $\bm q = (\sqrt{r^2-q_y^2}, q_y, 0, 0)$. We'll omit the argument $\bm q$ in $\bm p' =\Phi_{SM}(\bm p; \bm q)$ going forward.

Then, using the embedding $\iota: S^3 \hookrightarrow \mathbb R^4$, we can embed the tangent space in the ambient space $d\iota_{\bm p}:T_{\bm p}S^3 \hookrightarrow \mathbb R^4$, where it is given by the vectors $(0, v_1, v_2, v_3)$ for $v \in \mathbb R^3$.

In the tangent direction $v \in T_{\bm p}S^3$, the direction of change of $\bm p'$ is given by $\frac{\partial \Phi_{SM}(\bm p + t  v)}{\partial t}\rvert_{t=0} \in \mathbb R^4$. Thus, the Jacobian matrix, in standard coordinates of the tangent plane, is given by:
\[
J(\bm p)E_{\bm p} = \begin{pmatrix}
\frac{\partial \Phi_{SM}((1, t, 0, 0))}{\partial t}\rvert_{t=0}, 
\frac{\partial \Phi_{SM}((1, 0, t, 0))}{\partial t}\rvert_{t=0}, 
\frac{\partial \Phi_{SM}((1, 0, 0, t))}{\partial t}\rvert_{t=0}
\end{pmatrix}  \in \mathbb R^{4 \times 3}
\]
Via a computer algebra system, we can compute and simplify the change of volume to get:
\[
\sqrt{\det(E_{\bm p}^T J(\bm p)^T J(\bm p)E_{\bm p})}=\frac{(1-r^2)(1+r^2)^3}{(4 q_y^2 + (1-r^2)^2)^2}
\]

In the general case, the change of volume is also given by the above formula with $r = |\bm q|$ and $q_y^2 = r^2 - \langle \bm p, \bm q \rangle^2$.

\subsection{Derivations for projective convex gradient maps}

\paragraph{Invertibility}\label{a:diff_convex_map}

Following Sec. \ref{sec:flip-equi-diffeos} we assume that $\phi \colon \mathbb{R}^{d+1} \rightarrow \mathbb{R}$ is a strictly convex function with minimum at $\bm 0$. 
\begin{theorem}
    The projective convex gradient map $\Phi(\bm p) \colon S^{d} \rightarrow S^{d}, \bm p \mapsto \frac{\nabla_{\bm p} \phi(\bm p)}{\|\nabla_{\bm p} \phi(\bm p)\|}$ is a diffeomorphism.
\end{theorem}
\begin{proof}
    Let $\bm p \in S^{d}$ and $\bm p' = \Phi(\bm p)$ its image under the projective convex gradient map.
    Let furthermore $T_{\bm p}, T_{\bm p'}$ be the tangent spaces at $\bm p, \bm p'$, respectively. Furthermore let $\bm E_{\bm p},\bm E_{\bm p'} \in \mathbb{R}^{ (d+1) \times d}$ be ortho-normal bases of the two tangent spaces.
    Then it suffices to show that $\bm A := \bm E_{\bm p'}^T \bm J_{\Phi}(\bm p) \bm E_{\bm p} \in \mathbb{R}^{d \times d}$
    is a non-singular matrix with a non-zero determinant.
    Denote $\bm g(\bm p) := \nabla_{\bm p} \phi(\bm p)$.
    By using the chain rule, we first see that
    \begin{align}
        \bm J_{\Phi_{CG}}(\bm p) 
        &= \frac{\mathcal{H}_{\phi}(\bm p)}{\|\bm g(\bm p)\|} - \frac{\bm g(\bm p) \bm g(\bm p)^T }{ \|\bm g(\bm p)\|^3}  \mathcal{H}_{\phi}(\bm p) \\
        &= \left(\bm I  - \bm p' \bm p'^T \right) \frac{\mathcal{H}_{\phi}(\bm p)}{\|\bm g(\bm p)\|}  
    \end{align}

    We prove by contradiction:
    Assume $\bm A$ is singular and that $\bm v$ is a unit vector with $\bm A \bm v= \bm 0$ and denote $\bm w = \bm E_{\bm p}^T \bm v$. Since $\bm E_{\bm p}$ is a basis of $T_{\bm p}$ we have $\bm w \neq \bm 0$ and furthermore $\bm w \in T_{\bm p}$.
    
    Now because $\bm E_{\bm p}$ is full rank on its image we have $\bm J_{\Phi}(\bm p) \bm w = \bm 0$.

    
    Since $\phi$ is strictly convex $\mathcal{H}_{\phi}(\bm p)$ is a strictly positive definite matrix and thus we know that $\mathcal{H}_{\phi}(\bm p) \bm w \neq \bm 0$. Thus we have
    \begin{align}
        \bm 0 &= \left(\bm I  - \bm p' \bm p'^T \right) \frac{\mathcal{H}_{\phi}(\bm p)}{\|\bm g(\bm p)\|} \bm w \quad \Leftrightarrow  \quad \mathcal{H}_{\phi}(\bm p) \bm w = \bm p' \bm p'^T \mathcal{H}_{\phi}(\bm p) \bm w
    \end{align}
    And as such $\mathcal{H}_{\phi}(\bm p) \bm w \propto \bm p' \propto \bm g(\bm p) \implies \bm w \propto \mathcal{H}^{-1}_{\phi}(\bm p) \bm g(\bm p)$. Thus,  $\mathcal{H}^{-1}_{\phi}(\bm p) \bm g(\bm p) \in T_{\bm p} \implies \left(\mathcal{H}^{-1}_{\phi}(\bm p) \bm g(\bm p)\right)^T \bm p = 0$.

    Now set $\bm G = \mathcal{H}^{-1}_{\phi}(\bm p)$ keeping $\bm p$ fixed and define the function $\psi(\bm q) = \phi(\bm G \bm q)$.
    As $\bm G$ is strictly positive and $\phi$ is strictly convex with minimum $\phi(\bm 0)$ this new function $\psi$ is strictly convex with minimum $\psi(\bm 0)$ as well.

    Now  we can use the strict convexity of $\psi$ to get
    \begin{align}
        \psi(\bm 0) &> \psi(\bm p) + \nabla_{\bm p} \psi(\bm p)^T(\bm 0 - \bm p) \\
        &= \psi(\bm p) - (\bm G \nabla_{\bm p} \phi(\bm p) )^T \bm p \\
        &= \psi(\bm p) - \left(\mathcal{H}^{-1}_{\phi}(\bm p) \bm g(\bm p)\right)^T \bm p\\
        &= \psi(\bm p)
    \end{align}

    
    However, this contradicts $\bm0$ being the minimum of $\psi$.
\end{proof}

\paragraph{Parameterizing the potential $\phi$}
While $\phi$ could be modeled by any general input convex neural network \cite{amos2017} with minimizer $\bm 0$ we decided on a very simple implementation that worked well in practice and is fast to evaluate:

Let $\bm W \in \mathbb{R}^{d\times H}$, $\bm u \in \mathbb{R}_{>0}^{H}$, $\bm b \in \mathbb{R}_{>0}^{H}$ and $c \in \mathbb{R}_{>0}$. Then we define $\phi$ as 
\begin{align}
    \label{fig:gradient-map-icnn}
    \phi(\bm x; \bm W, \bm u, \bm b, c) := \bm u^T \texttt{softsign}(\bm W \bm x, \bm b) + c \cdot \bm x^T \bm x,
\end{align}
where 
\begin{align}
    \texttt{softsign}(\bm x, \bm b) := \log(\bm b + \cosh(\bm x)).
\end{align}
Here $H$ is a hyper-parameter that can control the complexity of the convex potential and thus its capability to model complicated multi-modal density.

\paragraph{Computing the volume element}\label{a:volume_convex_map}

For $d > 3$ computing the volume element, boils down to computing 
\begin{align}
    |\bm J_{\Phi}| = \det \bm E_{\bm p'}^T \bm J_{\Phi}(\bm p) \bm E_{\bm p}
\end{align}
by numeric means, which can become expensive and numerically unstable. For $d=3$ we can compute the full jacobian $\bm J_{\Phi}(\bm p)$, e.g., via \texttt{jax.jacrev} or \texttt{torch.autograd.functional.jacobian}. We can further compute the tangent bases $\bm E_{\bm p}, \bm E_{\bm p'}$ cheaply via a three-step Gram-Schmidt procedure relative to three standard basis vectors $\bm e_i, \bm e_j, \bm e_k$ in $\mathbb{R}^{4}$ which are independent of $\bm p$. Finally, we are only left with computing the determinant of the $3\times3$ matrix $\bm A =\bm E_{\bm p'}^T \bm J_{\Phi}(\bm p) \bm E_{\bm p}$ which can be done analytically and numerically stable, e.g., by computing $\det \bm A = (\bm a_{0} \times \bm a_{1})^{T} \bm a_{2}$.

\paragraph{Numerical inverse}

We can invert the projective convex gradient map by minimizing the residual $\ell \colon \mathbb{R}^{d + 1} \rightarrow \mathbb{R}$
\begin{align}
    \ell(\bm x) = \left\| \Phi_{CG}\left(\frac{\bm x}{\|\bm x\|}\right) - \bm p' \right\|^2
\end{align}
to get $\bm x^{*} = \arg\min_{\bm x \in \mathbb{R}^{4}} \ell(\bm x)$ and setting $\Phi^{-1}(\bm p') = \frac{\bm x^{*}}{\|\bm x^{*}\|}$. For the experiments in this paper, we minimized $\ell$ after training using the LBFGS solver \cite{liu1989limited} as implemented in \texttt{jaxopt} \cite{jaxopt_implicit_diff}. For $d=3$ and the potentials used in this work, the method converges up to an absolute error of $<0.00001$ in 10-15 iterations.

\paragraph{Affine quaternion flows from \citet{liu2023delving} are a special case}\label{a:liu_convex_map}
Let $\bm W \in \textrm{GL}(\mathbb{R}^4$).
The function $\phi(\bm p) = \bm p^T \bm W^T \bm W \bm p$ is strictly convex and satisfies all premises of the former theorem. Then 
\begin{align}
    \Phi(\bm p) &= \frac{\nabla_{\bm p} \phi(\bm p)}{\|\nabla_{\bm p} \phi(\bm p)\|} = \frac{\bm W \bm p}{\| \bm W \bm p \|}
\end{align}
is exactly the affine quaternion map as defined in \citet{liu2023delving}.

\change{
\subsection{Bundle flows} \label{a:covering_flows}
Let $\pi: S^3 \to SO(3)$ denote the fiber bundle projection, with the fibers $\{-1, 1\}$. Note that this is also a covering map.
Thus, for any point $p \in SO(3)$, there is an open neighbourhood $p \in U \subset SO(3)$, for which we can we can locally trivialize the fiber bundle, meaning in this case, we can pick an open set $\hat U \subset S^3$ and negation $-\hat U \subset S^3$, with two homeomorphisms $h_{\pm U}: U \xrightarrow{\sim} \pm \hat U$, such that $\{+\hat U, -\hat U\} = \pi^{-1} (U)$ and $\pi \circ h_{+ \hat U} = \pi \circ h_{- \hat U}=\mathrm{id}_{U}$.

Let $\hat f: S^3 \to S^3$ be a sign-equivariant mapping, meaning that $-\hat f(x)=\hat f(-x)$.
%
%
%
Then we can define a function  $f: SO(3) \to SO(3)$, which restricted to any neighbourhood $U \subset SO(3)$ is defined as:
\[
\restr{f}{U}=\pi \circ \hat f \circ h_{+\hat U} =\pi \circ \hat f \circ h_{-\hat U}
\]
Note that this construction does not depend on the choice of trivialization, and is therefore well-defined.
Then following diagram commutes, making the pair $(\hat f, f)$ a fiber bundle morphism.
\begin{equation}
\begin{tikzcd}[ampersand replacement=\&]
	{S^3} \& {S^3} \\
	{SO(3)} \& {SO(3)}
	\arrow["\pi", from=1-2, to=2-2]
	\arrow["f", from=2-1, to=2-2]
	\arrow["{\hat f}", from=1-1, to=1-2]
	\arrow["\pi"', from=1-1, to=2-1]
\end{tikzcd}
\label{diag:bundle-morphism}
\end{equation}

For a measurable space $X$, let $PX$ denote the space of measures on that space. For a measurable map $a: X \to Y$, let $a^*: PX \to PY$ denote the pushforward. For a stochastic map $b: X \to PY$, we also denote by $b^*: PX \to PY$ the induced map between measures.

Now, let $s: SO(3) \to PS^3$ be a stochastic map that is a stochastic section to the projection: $\pi^* \circ s^* = \mathrm{id}_{PSO(3)}$. Then consider the following diagram:
\[
\begin{tikzcd}[ampersand replacement=\&]
	{P\,SO(3)} \& {PS^3} \& {PS^3} \\
	\& {P\,SO(3)} \& {P\,SO(3)}
	\arrow["{\mathrm{id}}"', from=1-1, to=2-2]
	\arrow["{s^*}", from=1-1, to=1-2]
	\arrow["{\pi^*}"', from=1-2, to=2-2]
	\arrow["{\hat f^*}", from=1-2, to=1-3]
	\arrow["{\pi^*}", from=1-3, to=2-3]
	\arrow["{f^*}"', from=2-2, to=2-3]
\end{tikzcd}
\]
By the assumption that $s$ is a section to $\pi^*$, the left triangle commutes. The right square commutes because diagram \eqref{diag:bundle-morphism} commutes, which induces a commuting diagram of push-forwards. As both the left triangle and the right square commute, the outer two paths commute. This implies that the push-forward given by $f: SO(3) \to SO(3)$ equals first stochastically lifting to $S^3$, then applying the sign-equivariant $\hat f: S^3 \to S^3$, and projecting back to $SO(3)$. That computation does not depend on the choice of the particular stochastic section $s: SO(3) \to PS^3$. One choice could be to deterministically choose either point in each fiber.

More generally, this construction works for any bundle, if the map on the total space is a bundle morphism (satisfies diagram \eqref{diag:bundle-morphism}) and we can construct a stochastic section to the projection map.
}
\section{Details on tetrahedron experiments}\label{a:methane}

\paragraph{Control parameters of the chosen force-field}

The force field is given by the potential 
\begin{align}
    u(\bm x) = C \cdot \sum_{k=1}^{5} \sum_{d=1}^{3}(x_{kd} - c_{d})^{4},
\end{align}
with control parameters $\bm c = (0.09, -0.073, 0.), C= 136.98630$.

\paragraph{Setup of the simulation and data generation}

We use a methane molecule with reference coordinates given as
\begin{verbatim}
           x       y       z
     C    -0.037   0.090   0.000
    H1     0.070   0.090   0.000
    H2    -0.073   0.012   0.064
    H3    -0.073   0.073  -0.100
    H4    -0.073   0.184   0.035
\end{verbatim}

We then enforce the position of the carbon to be fixed by putting a position restraint to it. We fix the inner degrees of freedom by adding a bond constraint to the \texttt{CH} bonds and angle restraints to all possible \texttt{HCH} angles, fixing them to \texttt{109.47122°}.

We then run an OpenMM simulation using a Langevin integrator at 100K. We chose a time step of 1ps and only keep each 500th frame as a sample to ensure proper mixing. This results in a dataset of 50,000 samples.

\paragraph{Flow model}

\begin{figure}[t]
    \centering
\includegraphics{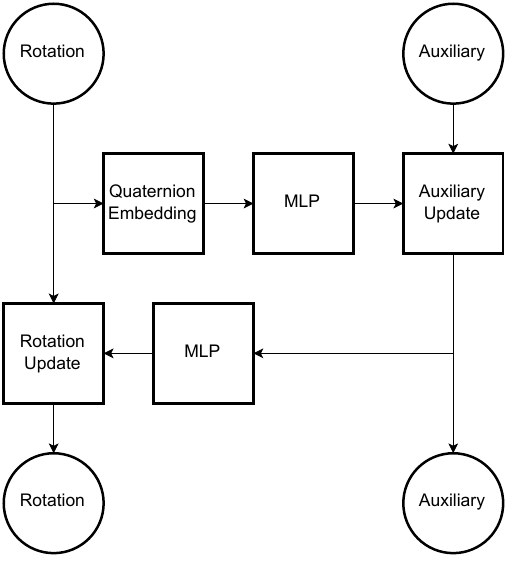}
    \caption{The coupling used for the tetrahedron experiment.}
    \label{fig:toy-model-flow}
    
\end{figure}

We use augmented normalizing flows \cite{huang2020augmented, chen2020vflow} and add auxiliary noise dimensions to our data as otherwise our dataset would only consist of one quaternion and as such could not be modeled with coupling layers. 
We use a two-dimensional unit normal distribution to model the auxiliary noise in data and latent space. We furthermore use an uninformed Von-Mises-Fisher (VMF) density with concentration parameter $2.5$ as base density for the rotations. To satisfy flip-invariance, we model this base density as a mixture of the location and its antipode. We then set up a two-layer coupling flow, coupling the rotation degree of freedom and the noise as depicted in Fig. \ref{fig:toy-model-flow}.

The conditioning functions producing the parameters of the flows are simple two-layer dense nets with 128 hidden units and GELU activation \cite{Hendrycks2016GaussianEL}. To ensure flip-invariance for the parameters of the auxiliary transformation we embed the conditioning quaternions using the following self-attention mechanism before feeding them into the dense nets (see Fig. \ref{fig:quat-embedding}):
\begin{itemize}
    \item Let $S \colon \mathbb{R}^{4} \rightarrow \mathbb{R}$ and $F \colon \mathbb{R}^{4} \rightarrow \mathbb{R}^{H}$ be linear layers where $H$ is some embedding dimension.
    \item Then we compute the quaternion embedding as
    \begin{align}
        G (\bm q) = \sum_{s \in \{-1, 1\}} \frac{\exp(S(s\cdot \bm q))}{\exp(S(s\cdot \bm q)) + \exp(S(-s\cdot \bm q))} \cdot  F(s \cdot \bm q)
    \end{align}
\end{itemize}

\begin{figure}
    \centering
    \includegraphics{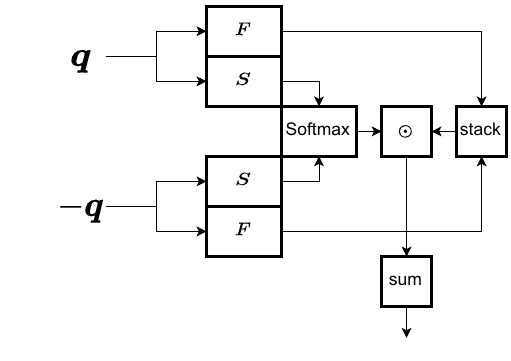}
    \caption{The flip-invariant embedding is used for the quaternions before feeding them into conditioning NN. The layers $F$ and $S$ are shared over inputs.}
    \label{fig:quat-embedding}
\end{figure}

For the auxiliary transformations we rely on simple real-NVP blocks \cite{dinh2016rnvp}. For the rotation transformation, we tried the following setups:
\begin{itemize}
    \item the affine flows of \citet{liu2023delving}.
    \item the symmetrized Moebius layers as presented in this work.
    \item three variants of the projective convex gradient maps as presented in this work using the potential defined in Eq. \ref{fig:gradient-map-icnn} setting $H=8, 32,$ and $ 128$ respectively.
\end{itemize}

Fig. \ref{fig:tetrahedron-result} in the main text shows the result for the convex potential with $H=128$. We show a comprehensive ablation of how the quality degrades when varying $H$ in figure \ref{fig:tetrahedron-result-extended}.\label{a:methane_param}

\begin{figure}[t]
    \centering
\includegraphics[width=0.9\linewidth]{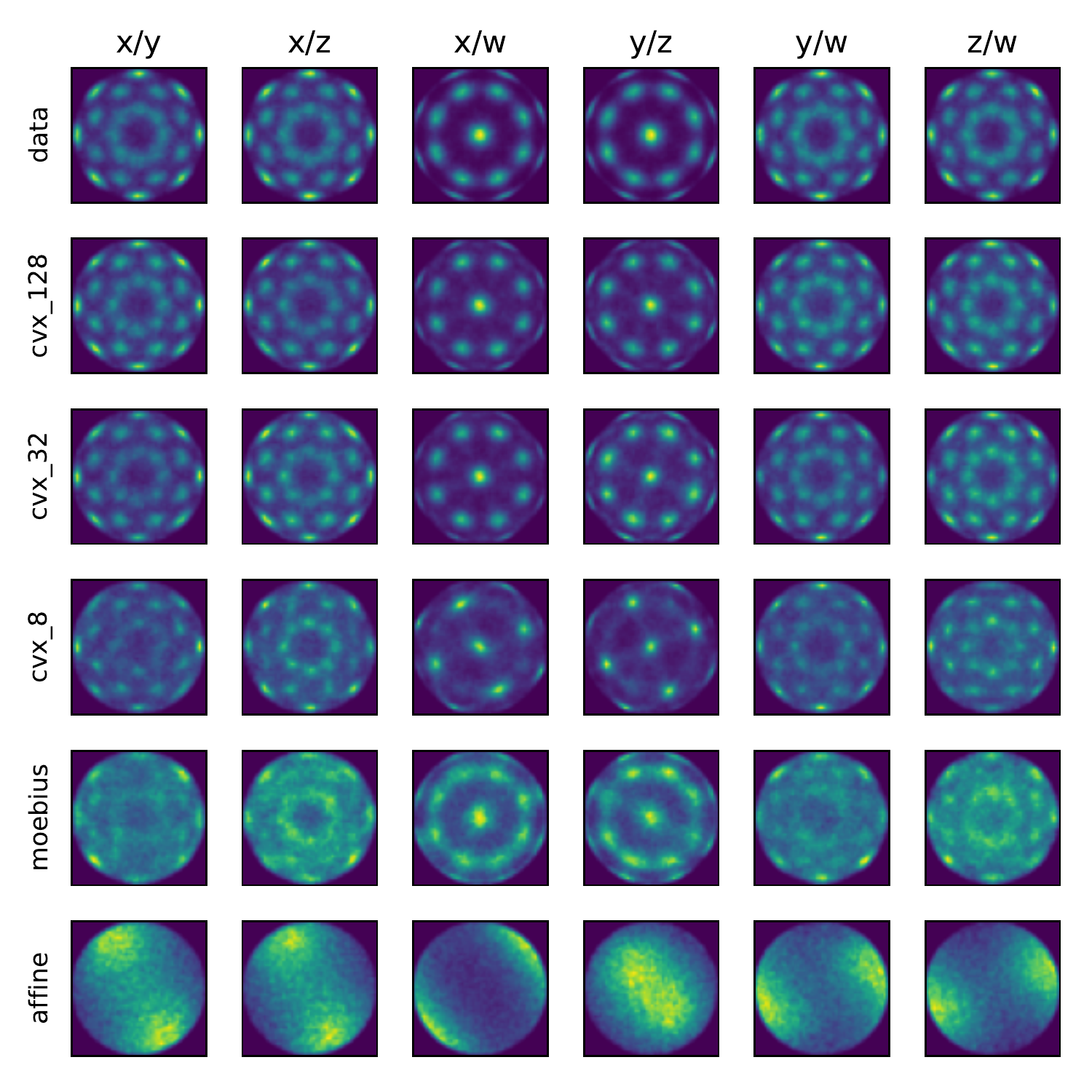}
    \caption{Full ablation for the methane experiment: in addition to the results shown in Fig. \ref{fig:tetrahedron-result} we show how the performance of the projective convex gradient map degrades by varying $H$ from $H=8$ to $H=128$.}
    \label{fig:tetrahedron-result-extended}
\end{figure}

\paragraph{Training}

By augmenting the data $\mu(\bm q)$ distribution with auxiliary noise $\bm z$, we obtain a new \emph{augmented data} density
\begin{align}
    \mu(\bm q, \bm z) = \mu(\bm q) \cdot \mathcal{N}\left(\bm z|\bm 0, \bm I\right).
\end{align}

To match dimensionality, we furthmore have to augment the base density as well, giving us
\begin{align}
    \mu_{0}(\bm q, \bm z) = \text{VMF}(\bm e_{0}, \kappa) \cdot \mathcal{N}\left(\bm z|\bm 0, \bm I\right).
\end{align}
Here $\bm e_{0} = (1, 0, 0, 0)$ and $\kappa = 2.5$.

Then the likelihood objective in eq. \ref{eq:maximum-likelihood} transforms into

\begin{align}
    \bm \theta_{ML} = \arg\min_{\bm \theta} \mathbb{E}_{\bm q, \bm z \sim \mu(\bm q, \bm z)}\left[ - \log \Phi(\mu_{0})(\bm q, \bm z; \bm \theta) \right].
\end{align}

We optimize this objective for $50,000$ steps for each candidate flow.
We used the ADAM optimizer \cite{kingma2014adam} with a batch size of $32$ and a learning rate of $0.0005$.

\section{Details on the Ice XI experiments}
We considered 3 different setups (see also Fig.~\ref{fig:ice-results}), in each case the temperature of the base distribution was T$_0$ 250~K.
We varied the number of water molecules $N$ and the temperature of the target distribution $T$:
\begin{itemize}
    \item[a)] N = 16, T = 100~K
    \item[b)] N = 16, T = 50~K
    \item[c)] N = 128, T = 100~K
\end{itemize}
Setup (a) and (b) use the same base distribution.

\paragraph{Details of the simulations and data generation process}
The initial configuration of ice XI was generated with the GenIce2 software \cite{Matsumoto2021}, using a single cell for the $N=16$ system and two cells per dimension for the $N=128$ one. We did not change the default size of the generated box.
We run molecular dynamics simulations with the OpenMM library \cite{eastman2017openmm} using the TIP4P-Ew rigid water force-field \cite{Horn2004} (cutoff length half the smaller box edge with a switching function for smooth interactions), and Langevin middle integrator \cite{Zhang2019a} with integration step of 1 fs and friction coefficient of 1 ps$^{-1}$. 
We chose the TIP4P-Ew water model because it is readily available in OpenMM, but similar results can be obtained with other rigid water models, such as TIP4P-ice \cite{Abascal2005}.
We run iterations of 500 MD steps and store a molecular configuration at each iteration.
All our MD simulations start with an equilibration run of 10,000 iterations, which is then discarded.

To generate the data used for training we run a MD simulation of 10,000 iterations for each of the two base distributions, and another 10,000 iterations MD for each of the two evaluation sets.

We note that ice XI is likely only metastable at the thermodynamic conditions that we consider, however, it is stable enough that we can run long MD simulations without observing any phase transitions, and thus it is perfectly suitable for our purposes.

\paragraph{Flow model}\label{a:ice_flow}
For all three experiments, we used the same coupling architecture, where we couple positions and rotations in a round-robin way over four iterations (see Fig. \ref{fig:ice-flow-model}).

The system is invariant with respect to global translations, thus the number of degrees of freedom associated to the positions of the molecules is not $3N$ but $3(N-1)$.
To account for this, we fix the position of one of the molecules (but not the rotation) and apply the flow only to the remaining $N-1$  \cite{wirnsberger2022}.

To share parameters and reuse local substructure, we model the conditioning networks with transformers \cite{vaswani2017attention} using multi-headed self-attention (see Figs. \ref{fig:conditioning-transformer}, \ref{fig:transformer-block}).
Since atoms in the crystal phase have well-specified positions we do not need to enforce strict permutation symmetry, e.g., as you would need to do when studying liquids.
Similar to positional encoding in language models, we encode the molecule index as a one-hot vector before feeding it into the first transformer block.
As we furthermore have to guarantee flip-invariance for the position conditioner, we use a modified attention mechanism where we stack rotations and the features of the last layer into different heads and then run the rotation logits through a square function to cancel its sign (see Fig. \ref{fig:attention-mechanism} a)).
In all experiments, we use two transformer blocks per flow layer each using 8 heads and 32 channels.

For updating the positions conditioned on the rotations we were using Moebius layers \cite{rezende2020normalizing}. For updating the rotations conditioned on the positions we used the symmetrized Moebius transforms described in the main text.

The flow is initialized to be the approximately the identity, so that at the beginning of training the mapped configurations have reasonable energies.
This is obtained by multiplying all the parameters of the coupling layers with a sigmoid function whose arguments are also learnable parameters.

The total number of trainable parameters is $290,896$ for system (a) and (b), and $7,458,896$ for system (c).

\begin{figure}[t]
    \centering
    \includegraphics{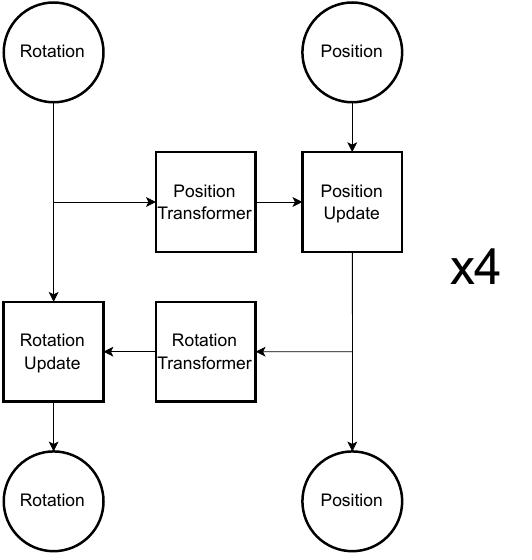}
    \caption{The flow architecture used for the ice XI experiments. The coupling layers are repeated four times.}
    \label{fig:ice-flow-model}
\end{figure}

\begin{figure}
    \centering
    \includegraphics{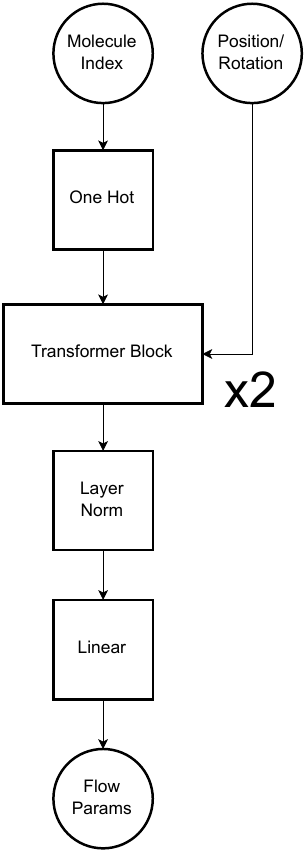}
    \caption{The conditioning layers for both the position and rotation conditioner. We use positional encoding and first embed the molecule index as a one-hot vector. Then we proceed with a stack of two transformer blocks. Finally, we decode the output into the flow parameters.}
    \label{fig:conditioning-transformer}
\end{figure}

\begin{figure}[t]
    \centering
    \includegraphics{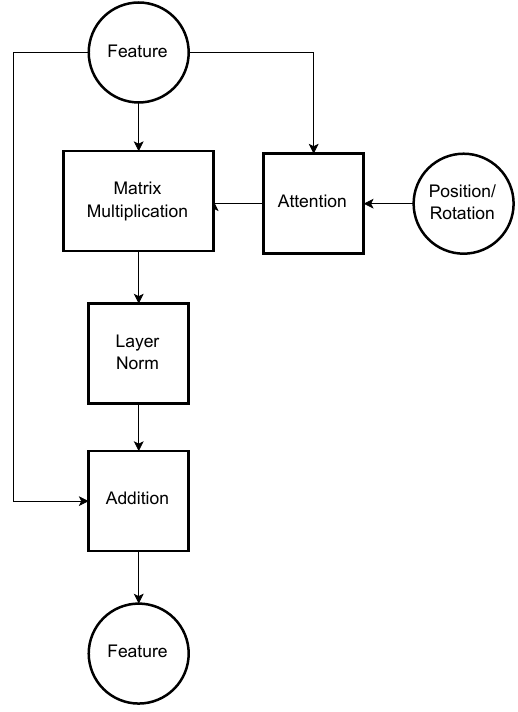}
    \caption{The transformer blocks follow the same structure for both positions and rotations and only differ in the attention mechanism used. As in usual transformer models, the features of the last layer together with current positions/rotations determine the self-attention matrix. We then multiply the last features with the computed attention matrix and add the result to the features of the last layer.}
    \label{fig:transformer-block}
\end{figure}

\begin{figure}[t]
    \centering
    \includegraphics{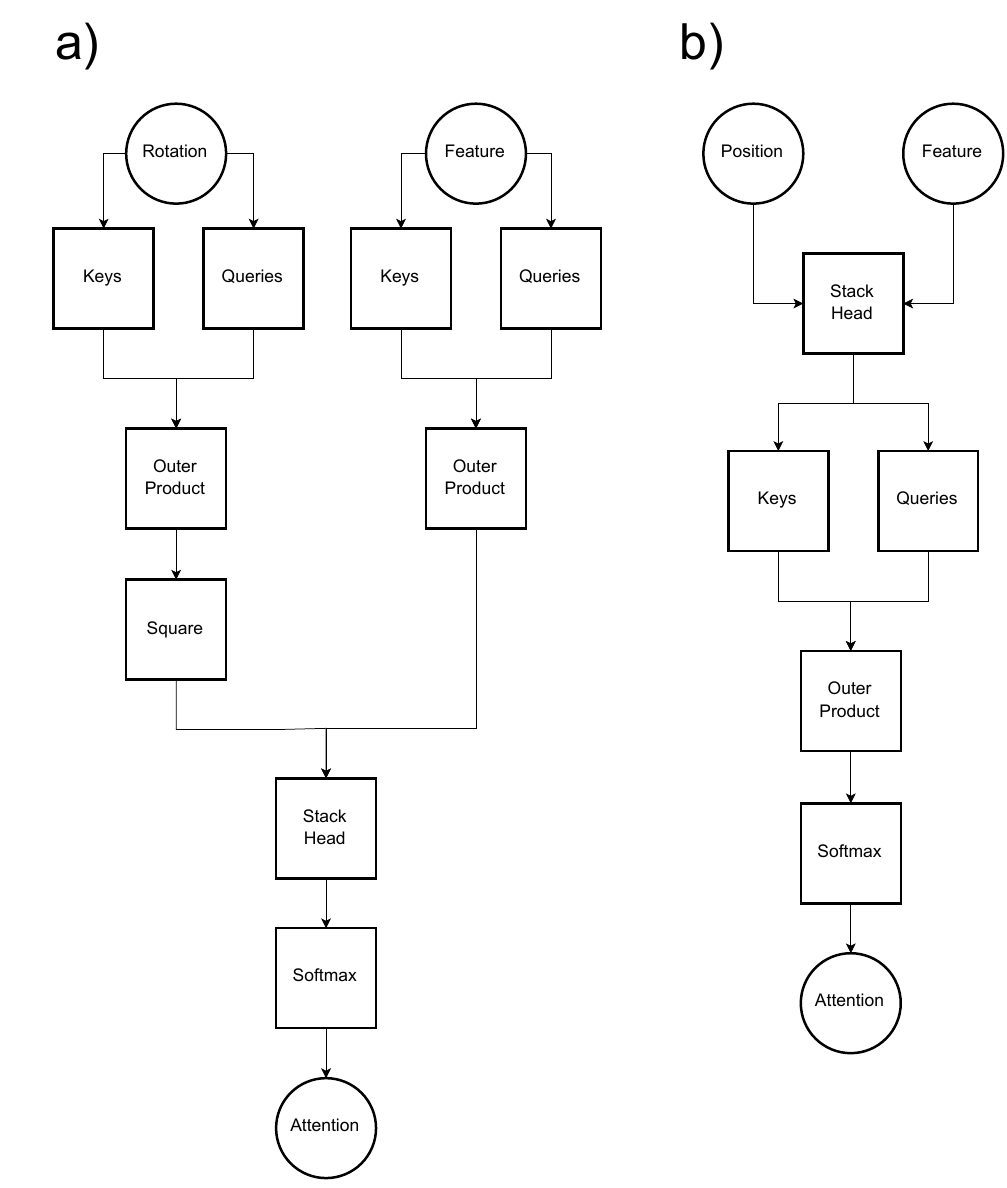}
    \caption{The attention mechanisms used within the transformer blocks. a) The attention mechanism for the position transformer: to preserve flip-invariance, we separate the information flow from rotations and features into separate heads. Then we apply a square function on the rotation logits to cancel the sign. Finally, we can stack along the head dimension and proceed with the softmax as usual. b) The  attention mechanism for the rotation transform boils down to a usual self-attention block, where we first concatenate features and positions before sending them through key and query blocks.}
    \label{fig:attention-mechanism}
\end{figure}

\paragraph{Training}

We train each model using ADAM \cite{kingma2014adam} for 1000 iterations per epoch over 10 epochs. We use a batch size of 32 and a cosine scheduler for the learning rate, which annealed the learning rate from 1e-3 in the first epoch to 1e-5 in the final epoch.
We used the same training setup for all the models.

\paragraph{Details on the evaluation}
Once the flow is trained, we apply it to the evaluation dataset to obtain the potential energy profiles and the oxygen-oxygen radial distribution functions shown in Fig.~\ref{fig:ice-results}.
To estimate the uncertainty for the LFEP estimator we use the bootstrapping method, and compute it 10 times by subsampling the evaluation data from the base distribution.

Similarly, we estimate the Kish effective sampling size \cite{kish1965} of the generated samples from 10,000 random points of the evaluation dataset.
The obtained averages and standard deviations over 10 estimates are: (a) $22.74 \pm 1.58\%$, (b) $6.90 \pm 1.36\%$, (c) $0.39 \pm 0.06 \%$.

The data for the reference MBAR $\Delta F$ calculations are obtained by running 10,000 MD iterations at various intermediate temperatures between the base and the target.
In total, we performed 5 MD runs for setup (a), 10 for (b), and 10 for (c).
The temperature ladder follows a geometric distribution.
The uncertainty over the MBAR calculations is estimated with bootstrapping, using the \texttt{pymbar} package (\url{https://github.com/choderalab/pymbar}).

Examples of the sampled atomistic configurations are shown in Fig.~\ref{fig:ice-config}, as 2D projections.

\paragraph{Computational cost} \label{a:ice_cost}
We did not optimise the computational efficiency of either our method or the MBAR reference, as this was not the aim of this work.
We expect the use of normalizing flows to bring a clear speed-up over more traditional MD methods, in cases where energy evaluation is more expensive, such as with interatomic potentials based on neural networks or on higher levels of theory.
In our experiment, the number of energy evaluations used for the MBAR estimate of (b) and (c) is $10  \times 10,000\times 500 = 50,000,000$, while for the LFEP estimate is $2 \times 10,000 \times 500 = 10,000,000$ for the MD sampling of the base distribution, plus $10 \times 1000 \times 32 = 320,000$ for the training.
Training on a GeForce GTX 1080 Ti took about 10 minutes for each of the small systems and 30 minutes for the larger one.

\begin{figure*}[t]
    \centering
    \includegraphics[width=\textwidth]{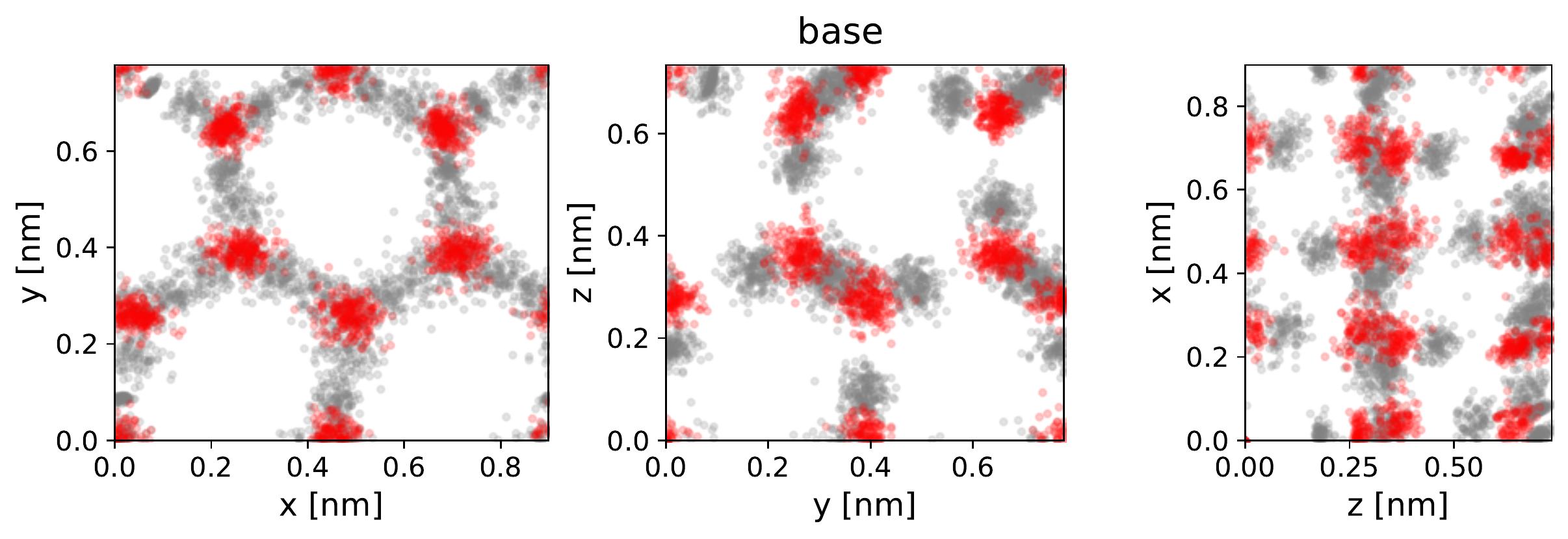}
    \includegraphics[width=\textwidth]{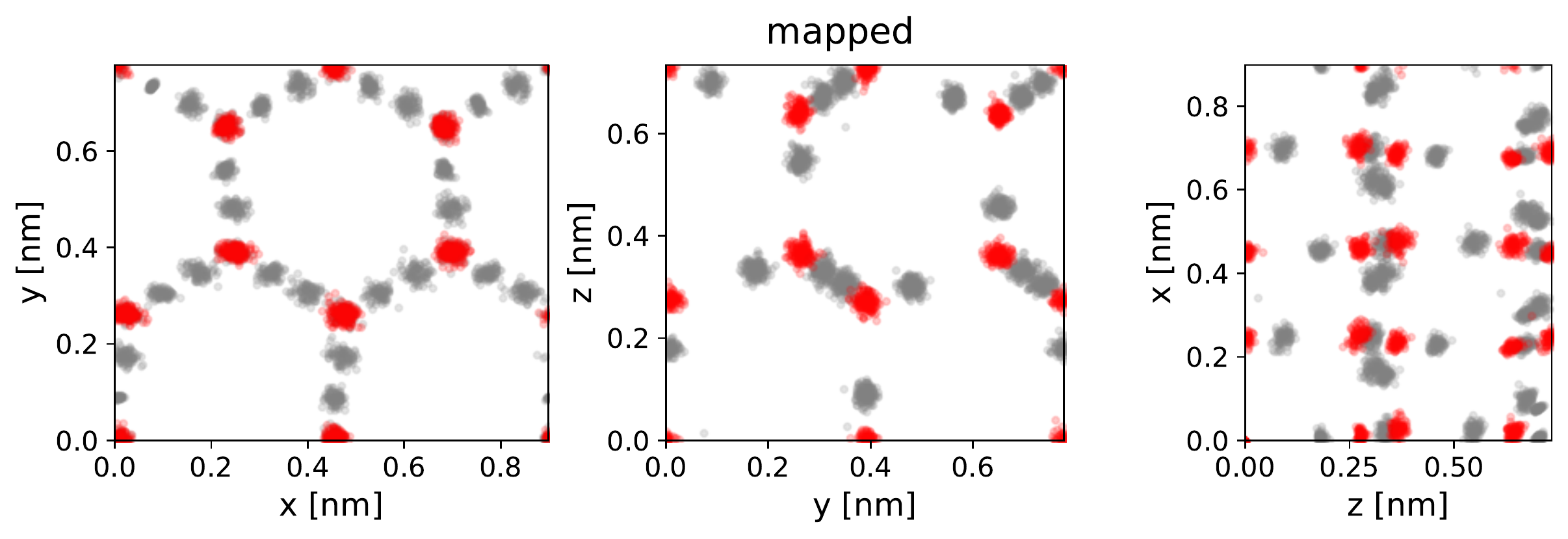}
    \includegraphics[width=\textwidth]{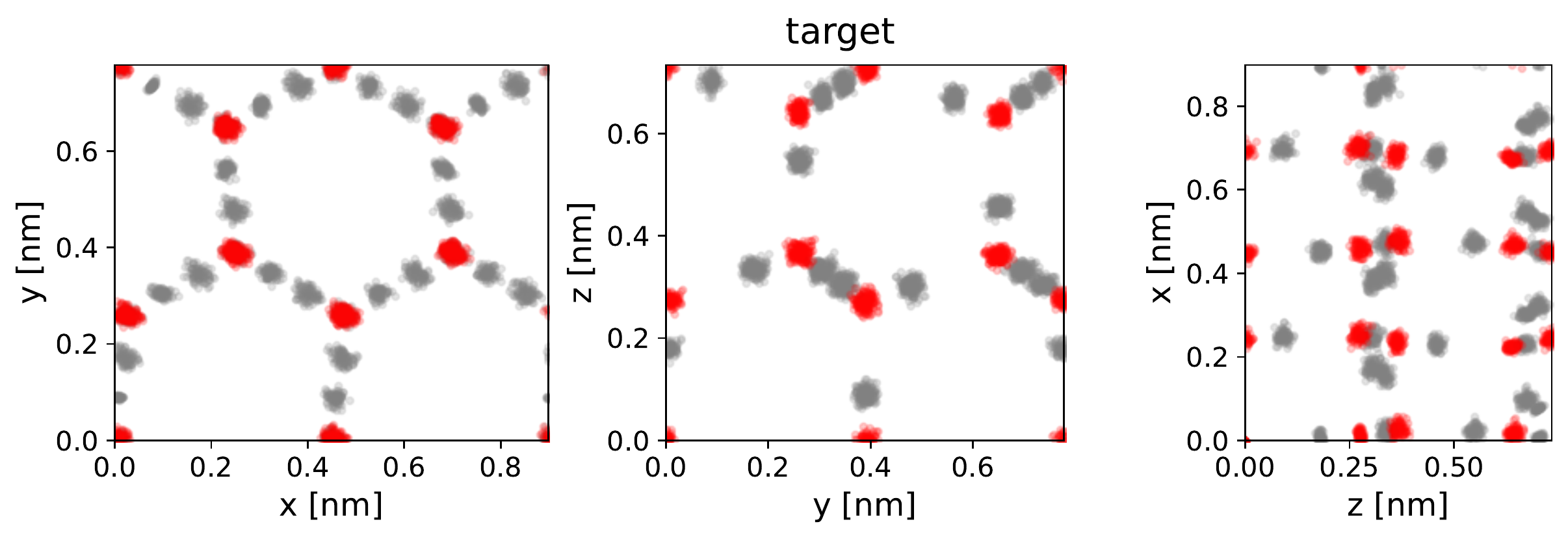}
    \caption{Atomistic configurations for the case of N = 16. 
    Each row shows the 2D projection of 100 different samples, with the oxygens colored in red and the hydrogens in gray.
    The first row contains samples from the base distribution at T$_0$ = 250~K, the second row shows how those configurations are mapped to the target by the NF, and the third row shows MD samples from the target distribution at T = 50~K. As the flows were purely trained on energies, these target samples were not seen during the training.
    }
    \label{fig:ice-config}
\end{figure*}


\end{document}